\documentclass[twoside]{article}

\usepackage[accepted]{aistats2026}
\makeatletter
\renewcommand{\AISTATS@appearing}{%
Proceedings of the \@conferenceordinal\,International Conference on Artificial
Intelligence and Statistics (AISTATS) \@conferenceyear, \@conferencelocation\@.
PMLR: Volume \@conferencevolume. Copyright \@conferenceyear\/ by the author(s). 
\textsuperscript{*}Equal contribution.
\textsuperscript{$\dagger$}Corresponding author.%
}
\makeatother

\usepackage[round]{natbib}

\usepackage[utf8]{inputenc}   
\usepackage[T1]{fontenc}      

\usepackage[dvipsnames]{xcolor} 
\usepackage{microtype}          
\usepackage{enumitem}

\usepackage{tikz}
\usetikzlibrary{arrows.meta,positioning,fit,backgrounds,calc}

\usepackage{amsmath,amssymb,amsfonts}
\usepackage{bbm} 
\usepackage{mathtools}          

\usepackage{amsmath,amsfonts,bm}

\def\eqref#1{equation~\ref{#1}}

\def\1{\bm{1}}

\DeclareMathAlphabet{\mathsfit}{\encodingdefault}{\sfdefault}{m}{sl}
\SetMathAlphabet{\mathsfit}{bold}{\encodingdefault}{\sfdefault}{bx}{n}

\def\gD{{\mathcal{D}}}

\newcommand{\E}{\mathbb{E}}

\usepackage{adjustbox}
\usepackage{amsthm}
\usepackage{wasysym}
\theoremstyle{plain}
\newtheorem{theorem}{Theorem} 

\theoremstyle{definition}
\newtheorem{definition}[theorem]{Definition}
\newtheorem{assumption}[theorem]{Assumption}

\theoremstyle{remark}

\usepackage{booktabs}           
\usepackage{multirow}           
\usepackage{wrapfig}            

\usepackage{algorithm}
\usepackage{algpseudocode}

\usepackage{nicefrac}           
\usepackage{pifont}             
\usepackage{comment}            
\usepackage{float}
\usepackage{caption} 

\usepackage{url}                
\usepackage[colorlinks=true,
            citecolor=MidnightBlue,
            linkcolor=magenta!60,
            urlcolor=blue]{hyperref} 
\usepackage[capitalize,noabbrev]{cleveref} 

\newcommand{\tr}{\textcolor{red}}
\newcommand{\tb}{\textcolor{blue}}
\newcommand{\tg}{\textcolor{green}}
\newcommand{\tgr}{\textcolor{gray}}
\newcommand{\tp}{\textcolor{magenta}}

\begin{document}

%
\runningtitle{Lyapunov‑Guided Self‑Alignment}

%
\runningauthor{S. Han, H. Kim, J. Lee}





\twocolumn[

\aistatstitle{Lyapunov‑Guided Self‑Alignment: Test‑Time Adaptation \\ for Offline Safe Reinforcement Learning}

\aistatsauthor{Seungyub Han \textsuperscript{*} \And Hyungjin Kim\textsuperscript{*} \And  Jungwoo Lee\textsuperscript{$\dagger$}}
\aistatsaddress{ Seoul National University }]




\begin{abstract}
Offline reinforcement learning (RL) agents often fail when deployed, as the gap between training datasets and real environments leads to unsafe behavior.  
To address this, we present SAS (Self-Alignment for Safety), a transformer-based framework that enables test-time adaptation in offline safe RL without retraining.  
In SAS, the main mechanism is self-alignment: at test time, the pretrained agent generates several imagined trajectories and selects those satisfying the Lyapunov condition.  
These feasible segments are then recycled as in-context prompts, allowing the agent to realign its behavior toward safety while avoiding parameter updates.  
In effect, SAS turns Lyapunov-guided imagination into control-invariant prompts, and its transformer architecture admits a hierarchical RL interpretation where prompting functions as Bayesian inference over latent skills.  
Across Safety Gymnasium and MuJoCo benchmarks, SAS consistently reduces cost and failure while maintaining or improving return.
\end{abstract}

\section{INTRODUCTION}\label{sec:intro}
Ensuring safety is a fundamental challenge for deploying RL in the real world. 
Recent advances in deep RL algorithms \citep{haarnoja2018soft, janner2019trust, lee2023spqr, eysenbach2022mismatched} 
have shown remarkable performance in simulation, but translating these gains to practice remains difficult, 
especially when a downstream controller must operate underactuated robots \citep{tedrake2009underactuated} 
or uncertain environments. 
Even small distribution shifts between training and deployment can cause highly unsafe behaviors, 
making safety a first-class requirement for RL deployment.

Offline RL \citep{kumar2020conservative} provides a promising step toward safer deployment, 
as it allows agents to be pretrained on large curated datasets 
such as D4RL \citep{fu2021drl}, RL-unplugged \citep{gulcehre2020rl}, and DSRL \citep{liu2023datasets}. 
By leveraging prior experience, offline RL agents can achieve high performance without direct online exploration. 
However, naively deploying a pretrained model is not sufficient: 
at test time, the environment often differs in unpredictable ways, 
and the resulting distribution shift may drive the agent into unsafe regions. 
Recent work has highlighted this challenge by modeling test-time uncertainty and partial observability 
\citep{pmlr-v162-ghosh22a, ghosh2021why}, 
showing that even for identical observations, transition can vary significantly due to latent dynamics.

\begin{figure*}[t]
    \centering
    \includegraphics[width=0.8\linewidth]{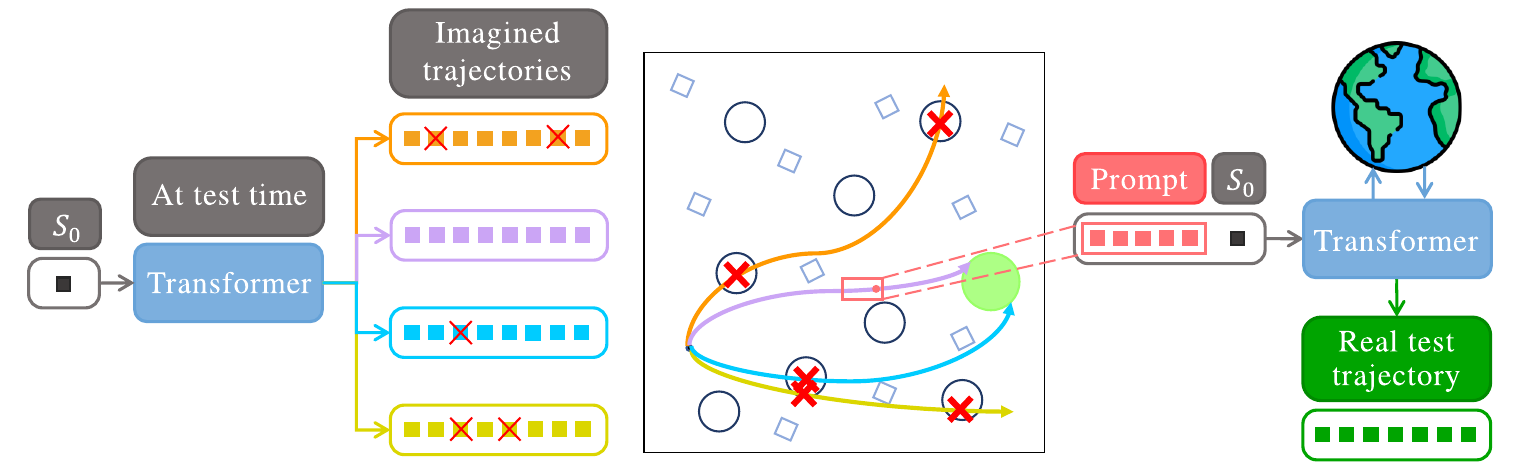}
\caption{
\emph{SAS overview}. From a fixed initial state, the transformer imagines multiple rollouts, flags risky state–action pairs using the Lyapunov condition with (\textcolor{red}{\ding{54}}), and extracts a safe segment as a prompt to guide the real test-time trajectory (hazards: black \textcolor{black}{$\bigcirc$}, blue \textcolor{blue}{$\Diamond$}; goal: green \textcolor{lime}{\ding{108}}).
}
    \label{fig:enter-label2}
\end{figure*}

To address this, prior approaches have pursued two directions. 
Belief-based adaptation augments the policy with latent belief variables to model uncertainty, 
but requires retraining from scratch for effective deployment. 
Conservative or constrained RL methods introduce penalties to reduce risk, 
but often sacrifice performance or require fine-tuning with additional cost signals. 
These limitations make it difficult to align pretrained offline RL agents 
with safety requirements at test time in a practical and scalable way.

In parallel, the field of large language models (LLMs) has demonstrated 
that pretrained models can be efficiently adapted to new objectives via \emph{self-alignment}, 
where models leverage their own reasoning ability to generate instruction-like prompts 
without human supervision \citep{sun2023principledriven}. 
Inspired by this idea, we ask: \emph{can offline RL agents be self-aligned for safety at test time, 
using only their pretrained distribution and without any retraining?}

We propose \emph{Self-Alignment for Safety (SAS)}, a transformer-based framework 
for safe test-time adaptation without additional training. 
SAS is based on two key insights: 
(i) Lyapunov stability can be enforced through occupancy measures, 
allowing the agent to identify control-invariant trajectories that avoid unsafe regions; and 
(ii) transformer-based RL naturally admits a hierarchical RL interpretation, 
where test-time prompting corresponds to Bayesian inference over latent high-level skills. 
These enable SAS to generate Lyapunov-guided prompts from imagined trajectories, 
and align agent behavior safely through in-context learning.
Our contributions are summarized as follows.
\begin{itemize}
    \item A new formulation of Lyapunov stability for offline RL as an occupancy-measure criterion.  
    \item A hierarchical RL interpretation of transformer-based RL via Bayesian inference over latent skills.  
    \item Our method, SAS uses Lyapunov-guided prompts for safe test-time adaptation.  
    \item Empirical evidence that SAS outperforms safe RL baselines, reducing cost and failure by up to two times while maintaining or improving return.  
\end{itemize}


\section{PRELIMINARIES}

\subsection{Problem Setting}
We consider a discounted Markov Decision Process (MDP)
$\langle \mathcal{S}, \mathcal{A}, r, c, \mathcal{P}, p_1, \gamma \rangle$,
where $\mathcal{S}$ and $\mathcal{A}$ are the state and action spaces.
Reward and cost are
$r, c: \mathcal{S}\times\mathcal{A} \to \mathbb{R}$,
and the transition kernel $\mathcal{P}: \mathcal{S}\times\mathcal{A} \to \Delta(\mathcal{S})$ gives the next-state distribution.
$p_1$ is the initial state distribution and $\gamma \in (0,1)$ the discount factor.
We assume the per-step cost is uniformly bounded:
\begin{equation}\label{eq:cost_bound}
    0 \;\le\; c(s,a) \;\le\; C_{\max}, \quad \forall (s,a)\in\mathcal{S}\times\mathcal{A}.
\end{equation}
This assumption is standard in constrained MDP (CMDP) and ensures the discounted cost is well-defined.

We also introduce a latent skill space $\mathcal{Z}$ and consider hierarchical policies. 
The high-level policy $\pi^{\textnormal{high}}_{\theta}(z_t \mid s_t, z_{t-1})$ 
selects a new skill $z_t \in \mathcal{Z}$ given the current state and the previously chosen skill, 
and the low-level policy $\pi^{\textnormal{low}}_{\phi}(a_t \mid s_t, z_t)$ 
executes environment actions conditioned on both the state and the selected skill.
For simplicity, we use the notation $\pi$ to denote the overall policy, 
which encompasses both the high-level policy $\pi^\text{high}$ 
and the low-level policy $\pi^\text{low}$.

\paragraph{Hierarchical RL as Probabilistic Inference.}
We interpret hierarchical RL as a probabilistic graphical model (PGM). 
The joint distribution $p(\tau)$ over a trajectory 
$\tau = (s_{1:T}, a_{1:T}, z_{1:T})$ factorizes as
\begin{equation*}
    p_1(s_1) \prod_{t=1}^{T}
    \mathcal{P}(s_{t+1}\mid s_t,a_t)\,
    \underbrace{p(a_t\mid s_t,z_t)}_{\pi_\phi^{\textnormal{low}}}\,
    \underbrace{p(z_t\mid s_t,z_{t-1})}_{\pi_\theta^{\textnormal{high}}},
\end{equation*}
where $z_t$ is a latent skill, $p_1(s_1)$ is the initial state distribution, 
and $\mathcal{P}$ denotes the transition kernel. 



We remark that for $t=1$, the term $z_{0}$ is undefined; here we adopt a mild abuse of notation, treating $\pi^{\textnormal{high}}(z_1 \mid s_1,z_0)$ as $\pi^{\textnormal{high}}(z_1 \mid s_1)$ (equivalently, $z_0$ can be regarded as drawn from a prior).
Following the control-as-inference framework \citep{levine2018rlaspi}, 
we also introduce the \emph{optimality variable} 
$\mathcal{O}_t \in \{0,1\}$, a binary observable that encodes reward-based 
optimality. 
Later, we will define additional binary observables for Lyapunov or safety 
conditions. 
While all such observables are binary indicators, they differ in their role: 
$\mathcal{O}_t$ connects rewards to the PGM, whereas the additional variables 
capture safety-related constraints. 
This unified PGM view, including the observable nodes, is illustrated in \cref{fig:pgm1}.


\begin{figure}[t]
\centering

\tikzset{
  node/.style={circle,draw,minimum size=22pt,inner sep=0pt,thick,font=\large},
  var/.style={node,fill=white},        
  skill/.style={node,fill=white},      
  opt/.style={node,fill=gray!20},      
  >={Latex[length=2mm]},
  edge/.style={->,thick},
  faint/.style={->,line width=0.9pt},
}

\begin{tikzpicture}[x=2.6cm,y=1.6cm]
  \def\yO{2.0}   
  \def\ya{1.1}   
  \def\ys{0.3}   
  \def\yz{-0.5}  

  \def\dx{0.70}  
  \def\dyA{0.65} 
  \def\dyO{1.20} 

  \foreach \t/\x in {1/0,2/1,3/2} {
    \node[opt]   (O\t) at (\x,\yO) {$\mathcal{O}_{\t}$};
    \node[var]   (a\t) at (\x,\ya) {$a_{\t}$};
    \node[var]   (s\t) at (\x,\ys) {$s_{\t}$};
    \node[skill] (z\t) at (\x,\yz) {$z_{\t}$};
  }

  \node[fit=(z1)(z2)(z3), rounded corners, draw, dashed, inner sep=6pt] {};

  \foreach \t in {1,2,3} {
    \draw[edge] (s\t) -- (a\t);     
    \draw[edge] (a\t) -- (O\t);     
    \draw[edge] (s\t) -- (z\t);     
  }

\foreach \t in {1,2,3} {
  \draw[edge] (s\t) to[out=140, in=220, looseness=1.1] (O\t);

  \draw[edge] (z\t) to[out=140, in=220, looseness=1.0] (a\t);
}

  \foreach \t/\tp in {1/2,2/3} {
    \draw[edge] (s\t) -- (s\tp);                   
    \draw[edge] (z\t) -- (z\tp);                   
    \draw[faint] (a\t) to (s\tp);     
  }

\end{tikzpicture}

\caption{Hierarchical RL as probabilistic inference: a hidden Markov model with latent skills $z_t$, the optimality variable $\mathcal{O}_t$ linking rewards, and observables for safety conditions.}
\label{fig:pgm1}
\end{figure}
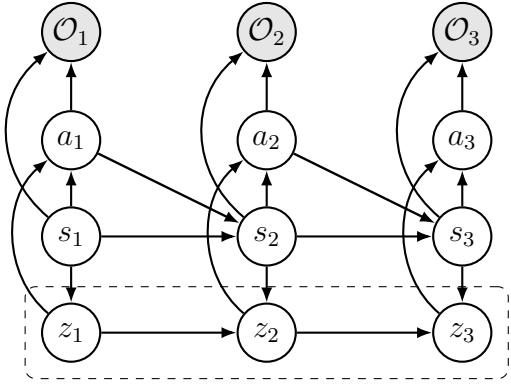

\subsection{Occupancy Measure for Safe RL}

Transformer-based RL can autoregressively predict actions and, when paired with a generative model such as a VAE \citep{kingma2013auto}, also generate plausible future states. This allows us to construct imagined trajectories, i.e., rollouts that simulate how the agent would behave under the learned dynamics. In our setting, the VAE decoder serves as a \emph{world model} and is integrated into a Decision Transformer \citep{chen2021decision} (see \cref{fig:dt}).

Given this pretrained transformer with its integrated world model, we define the \emph{model-based occupancy measure} $\hat{\rho}^\pi$, which approximates the true occupancy measure by using imagined trajectories generated from the learned model. Formally, let $\hat{\mathcal{P}}(s' \mid s,a)$ denote the learned transition model (the VAE decoder) and $\pi(a \mid s)$ the policy induced by the transformer, where we write $\pi(a\mid s)=\sum_{z}\pi^{\textnormal{low}}(a\mid s,z)\,\pi^{\textnormal{high}}(z\mid s, z_{-1})$ and omit $z$ for brevity. Then \footnote{Unormalized version. The $(1-\gamma)$ factor is omitted since it can be canceled in CMDP objectives.}
\begin{align}\label{eq:occupancy_measure}
\hat{\rho}^\pi(s,a)
= \sum_{t=1}^{\infty}\gamma^{t-1}\,
\underset{\pi,\hat{\mathcal P}}{\mathbb{P}}\!\big(s_t=s,\ a_t=a \mid s_1 \sim p_1\big).
\end{align}
We use $\mathcal{P}$ for the true environment and $\hat{\mathcal{P}}$ for the learned world model dynamics.

\paragraph{Connection to Safe RL.}
We formulate safe RL as a constrained optimization problem:
\begin{align}\label{eq:cmdp}
    \max_{\pi} \; J_{R}(\pi) \quad \text{s.t.} \quad J_{C}(\pi) \leq d,
\end{align}
where 
$J_{R}(\pi)=\mathbb{E}_{(s,a)\sim \rho^\pi}[\,r(s,a)\,]$ and 
$J_{C}(\pi)=\mathbb{E}_{(s,a)\sim \rho^\pi}[\,c(s,a)\,]$.
This form shows that the optimal policy $\pi^*$ must assign low occupancy to high-cost regions; in other words, safety constraints act directly on the distribution of visited state–action pairs.

\paragraph{Generalization with Offline Data.}
Offline datasets are often heterogeneous, containing both expert and medium-quality trajectories.  
Following \citet{rashidinejad2021bridging}, we introduce the best concentrability coefficient $D_\text{conc}\!\ge\!1$, the smallest constant:
\begin{equation*}
    \hat{\rho}^{\pi^*}(s,a) \;\le\; D_{\text{conc}}\, \rho_{\textnormal{data}}(s,a),
    \quad \forall (s,a)\in \mathcal{S} \times \mathcal{A}
\end{equation*}
and $\rho_\text{data}>0.$
A smaller $D_{\text{conc}}$ indicates that the dataset provides good coverage and thus tighter safety guarantees, 
whereas a larger $D_{\text{conc}}$ reflects sparse or biased data and results in weaker guarantees.
Combining this bound with the cost boundedness in \cref{eq:cost_bound},  
we obtain the conservative sufficient condition
\begin{equation}\label{eq:R_rho}
    \mathcal{R}_\rho \;=\; \Bigl\{ (s,a) \;\Big|\; \hat{\rho}^\pi(s,a) \;\ge\; \tfrac{d}{C_{\max} D_{\text{conc}}} \Bigr\},
\end{equation}
which yields a tractable feasible region for satisfying the CMDP constraint in \cref{eq:cmdp}.

\paragraph{Interpretation.}
\cref{eq:cmdp} reveals an inverse relationship between occupancy and cost:  
a safe policy must avoid placing significant probability mass in high-cost regions.  
$D_{\text{conc}}$ adapts this principle to the offline setting by accounting for dataset coverage.  

\subsection{Lyapunov Density and Control-Invariant Sets for Safety}

Building on the feasible set $\mathcal{R}_\rho$, 
we now turn to the \textbf{Lyapunov Density Model (LDM)} \citep{kang2022lyapunov}, 
which formalizes safety as density-based monotonicity 
and provides a novel way to define control-invariant sets. 
Practically, restricting the agent to actions within $\mathcal{R}_\rho$ ensures 
that the safety constraint in \cref{eq:cmdp} is satisfied, 
and \cref{fig:enter-label2} illustrates how imagined trajectories are filtered 
using both occupancy measures and Lyapunov conditions. 
Unlike classical state-only control Lyapunov functions (CLFs), 
LDM is formulated in a $(s,a)$-based manner: the Lyapunov surrogate $G(s,a)$ is 
\textbf{learned from offline data}.
LDM posits a Lyapunov-like function 
$G:\mathcal{S}\times\mathcal{A}\to\mathbb{R}_{\ge0}$ 
satisfying:
\begin{enumerate}
    \item $G(s_e,a_e)=0$ for an equilibrium $(s_e,a_e)$
    \item $G(s,a)>0$ for all $(s,a)\neq (s_e,a_e)$
    \item $G(s,a)\,\ge\, \gamma\mathbb{E}_{s'\sim \mathcal{P}(\cdot|s,a)}\!\big[ \min_{a'} G(s',a') \big]$.
\end{enumerate}
These conditions ensure that $G$ is \emph{non-increasing} along feasible trajectories:  
whenever $(s',a')$ is the next state–action pair reached from $(s,a)$ by the following \emph{greedy policy}, 
the value of $G$ cannot increase.
\textbf{LDM backup operator} is defined as 
a Bellman-style recursion adapted to enforce Lyapunov density conditions:
\begin{align*}
    \big(\mathcal{T}_{\text{LDM}}G\big)(s,a) 
    &= \max \Big\{ -\log \hat\rho^\pi(s,a), \\
    &\qquad \gamma\, \mathbb{E}_{s' \sim P(\cdot|s,a)} \big[\min_{a'} G(s',a')\big] \Big\}
\end{align*}

where the first term ties G to the estimated density, and the second term enforces discounted
Lyapunov descent under the learned dynamics.
The $\max$ selects the more conservative of the two constraints.

\paragraph{Control-Invariant Set for Safety.}
Define the control (forward)-invariant set as
\begin{equation}
    \mathcal{R}_{c} = \{(s,a)\mid G(s,a)\le c\}, \quad c>0.
\end{equation}
$G$ is non-increasing along feasible trajectories.  
Thus, a \emph{greedy policy} such as $\arg\min_a G(s,a)$ cannot escape from $\mathcal{R}_c$, 
and every trajectory remains inside while asymptotically converging to the equilibrium~point.

\section{FROM OFFLINE RL TO LYAPUNOV-BASED SAFETY EVALUATION}
\label{sec:from_offline_rl}
In this section, we establish a principled transition from offline RL to Lyapunov-based safety evaluation. 
While offline RL provides a pretrained policy distribution, it does not guarantee safety at deployment, 
as trajectories may still enter unsafe regions due to limited coverage. 
To address this, we reinterpret Lyapunov stability in terms of occupancy measures estimated from offline RL models, 
and introduce our Lyapunov model $G_{\text{SAS}}$ as a principled way to evaluate and enforce safety. 
This section formalizes the connection between safe RL and occupancy measures, presents our core Lyapunov model, 
and develops a probabilistic inference framework that enables trajectory-level search for safety.

\subsection{Lyapunov Model for Offline Safe RL}

Our central contribution in this section is to extend the Lyapunov Density Model (LDM) with occupancy measures, 
yielding a Lyapunov model tailored for offline safe RL. 
We define the energy function based on the estimated occupancy measure $\hat{\rho}^\pi$ in \cref{eq:occupancy_measure}:

\[
E(s,a) = - \log \hat{\rho}^\pi(s,a),
\]
which serves as an energy function that provides an upper bound for the control-invariant set. 
Controlling with the learned LDM guarantees that the agent stays within this set. 
Extending this idea, we define our Lyapunov model for safety evaluation as
\begin{align}\label{eq:our_lyapunov}
&G_\textnormal{SAS}(s,a) = \min_{\pi}\max_{t'} \Big[\,E(s_{t'},\pi(s_{t'})) - E(s,a)\,\Big] \\
&= \log \hat{\rho}^{\pi}(s,a)
   - \underbrace{\max_{\pi\in\Pi}\; \min_{t'\in\{1,\dots,T\}}
      \log \hat{\rho}^{\pi}\big(s_{t'},\pi(s_{t'})\big)}_{\displaystyle \text{related to $R_\rho$ in Eq.~(\ref{eq:R_rho})}} . \nonumber
\end{align}

By construction, $G_{\text{SAS}}$ maximizes the lowest occupancy encountered along any trajectory, 
thereby preventing the policy from collapsing into unsafe, low-density regions. 
This induces a control-invariant set:
\begin{equation*}
\mathcal{R}_G^\textnormal{SAS} 
= \bigl\{ (s_t,a_t) \,\big|\, 
   G_\textnormal{SAS}(s_t,a_t) \le -\log \tfrac{d}{C_{\max} D_\text{conc}} \bigr\}.
\end{equation*}

\paragraph{Remark.}
$G_{\text{SAS}}$ refines the offline RL model by enforcing Lyapunov stability in terms of data-supported occupancies. 
This guarantees that imagined trajectories remain within a safe invariant set without retraining or fine-tuning.
Two key implications follow from these properties:  
\textbf{(1) Conservative behavior.} Since $G_\text{SAS}$ is non-increasing along feasible trajectories, 
policies are naturally restricted to high-density regions of the offline data, 
reducing the risk of out-of-distribution actions.  
\textbf{(2) Test-time self-alignment.} 
In our framework, densities and occupancy measures are estimated from offline data and imagined rollouts 
(via transformer-generated trajectories), so $G_\text{SAS}$ serves as a prior that guides prompt selection and action choices 
toward the safe region $\mathcal{R}_G^\text{SAS}$. 
We will elaborate on how this prior enables test-time self-alignment in \cref{sec:sas}.

\subsection{Lyapunov Condition as Probabilistic Inference}\label{sec:lyapunov_as_pi}

Additional training on top of the pretrained model is computationally inefficient. 
Instead, we reformulate our Lyapunov model for safety evaluation, $G_\textnormal{SAS}(s,a)$ in \cref{eq:our_lyapunov}, as a
probabilistic inference that can be integrated into the transformer's in-context learning. 

By introducing observable indicators of the Lyapunov condition, safety verification can then be 
cast as maximizing their log-likelihood. This reformulation provides a principled criterion to 
rank and select safe \emph{imagined} trajectories without any additional retraining.
To formalize this reformulation, we first define observable indicators that encode whether the Lyapunov condition is satisfied along a trajectory.

\begin{definition}[Lyapunov-Condition Observables]\label{def:lyapunov_obs}
To make the Lyapunov condition verifiable, we introduce two indicator observables:
\begin{align*}
\mathcal{U}_t &= \1\!\left\{\,G_{\textnormal{SAS}}(s_t,a_t) > 0\,\right\}, \\
\mathcal{V}_t &= \1\!\left\{\,G_{\textnormal{SAS}}(s_t,a_t) - 
                     G_{\textnormal{SAS}}(s_{t+1},a_{t+1}) \ge 0\,\right\}.
\end{align*}
\end{definition}

Using these observables, our Lyapunov model in \cref{eq:our_lyapunov} can be shown to be equivalent to maximizing their log-likelihood, as stated in the following theorem.

\begin{theorem}\label{thm:lyapunov_gm}
With the observables in Definition~\ref{def:lyapunov_obs}, the Lyapunov formulation in
\cref{eq:our_lyapunov} is (up to an additive constant) equivalent to
\begin{equation*}
\max_{\tau}\; \frac{1}{T}\sum_{t=1}^T 
  \big[ \log P(\mathcal{U}_t{=}1 \mid \tau) 
      + \log P(\mathcal{V}_t{=}1 \mid \mathcal{U}_t{=}1,\tau) \big].
\end{equation*}
\end{theorem}

Here, $T$ denotes the trajectory horizon and $\tau$ an imagined trajectory generated by the pretrained model. 
The full proof is given in \cref{sec:proof}. 

This theorem establishes that the Lyapunov condition can be reformulated as a probabilistic inference objective,
which allows safety to be assessed and enforced directly on imagined trajectories.
To compute $G_{\textnormal{SAS}}$ while simultaneously selecting an optimal trajectory,
we relax the problem into the iterative imagination-based procedure described in \cref{alg:main} of \cref{sec:sas}.

\section{LYAPUNOV-GUIDED SELF-ALIGNMENT}\label{sec:sas}
In \cref{sec:from_offline_rl}, we introduced the Lyapunov condition as a criterion for 
control-invariant sets in offline RL. 
Here, we ask: how can this condition be used to realign a pretrained agent 
\emph{without retraining}? 
Our key idea is to view Lyapunov evaluation as an inference problem over imagined trajectories, 
so that a pretrained model can be guided toward safer behavior 
via self-generated prompts.
We consider two complementary axes:
\textbf{(1) Occupancy-based evaluation:} filtering trajectories with low estimated density, 
thereby mitigating error accumulation along the \emph{value horizon}. 
\textbf{(2) Trajectory-based prompting:} Consider the decay of $G_\text{SAS}$ as a binary observable, 
and recycling safe segments as in-context prompts to shorten the \emph{policy horizon.}  
Together, these axes constrain imagined rollouts and enable prompt-based alignment.

\subsection{Self-alignment via in-context learning}
Our procedure, \textbf{SAS (Self-Alignment for Safety)}, integrates Lyapunov filtering 
with transformer imagination through three steps:
(1) evaluate imagined rollouts with $G_{\textnormal{SAS}}(s,a)$,  
(2) select safe fragments within $\mathcal{R}_G^{\textnormal{SAS}}$,  
(3) insert them back as prompts.  
This jointly reduces value and policy horizons and enables safe test-time adaptation 
\emph{without parameter updates}. 
We next formalize the inference view underlying SAS.

When aligning a transformer-based model to a downstream task,
the model adapts by conditioning on a prompt composed of demonstration examples.
This capability, known as \emph{in-context learning}, 
can be interpreted as implicit Bayesian inference over latent concepts \citep{xie2021explanation}.
Formally, the posterior predictive distribution $p(\texttt{output} \mid \texttt{prompt})$ is
\begin{align*}
    \int p(\texttt{output} \mid \texttt{prompt}, \theta)\, p(\theta \mid \texttt{prompt})\, d\theta,
\end{align*}
where $\theta$ denotes a latent concept that parameterizes the hidden Markov model 
$p(\texttt{output} \mid \texttt{prompt}, \theta)$. 
Intuitively, conditioning on a prompt selects a particular $\theta$, 
so that $p(\texttt{output} \mid \texttt{prompt},\theta)$ generates aligned behavior.
In our framework, we draw an analogy between this latent concept $\theta$ 
and the parameter of a high-level policy in hierarchical RL,
with the low-level policy handling primitive actions.
This perspective motivates the probabilistic formulation of SAS in the following.

\subsection{In-context Bayesian Inference for Prompt-based Policy Extraction}
\label{sec:hrl}

We now show how the policy that generates prompts in our self-alignment framework 
can be identified within the transformer through an in-context Bayesian inference view. 
This enables us to connect prompt-based self-alignment with Bayesian inference 
and to derive probabilistic guarantees on the resulting trajectories.

Here, unlike in the LLM setting where $\theta$ denotes a latent concept,  
we interpret $\theta$ as the parameter of the high-level policy  
$\pi^{\textnormal{high}}_\theta$ in our hierarchical RL formulation.  
At test time, the first latent skill $\mathbf{z}_1^{\textnormal{test}}$ is sampled 
along with a trajectory $\tau$ starting from the initial state 
$\mathbf{s}_1^{\textnormal{test}}$, conditioned on a prompt 
$\mathbf{p}_{1:K}$, which is a $K$-length demonstration prefix of
state–action pairs $(\mathbf{s}_{-K+1}, \mathbf{a}_{-K+1}, \dots, \mathbf{s}_0, \mathbf{a}_0)$ 
(see Figure~\ref{fig:pgm1}).  
Intuitively, such a prompt selects a feasible high-level policy 
$\pi^{\textnormal{high}}_\theta$ from the space implicitly encoded in the offline dataset $\mathcal{D}$.
Our goal is to recover the safe high-level policy $\pi^{\textnormal{high}}_{\theta^*}$ 
consistent with the demonstration prompt.  
To formalize this inference, we introduce the Lyapunov observables $\mathcal{U}_t$ and $\mathcal{V}_t$ 
from \cref{sec:lyapunov_as_pi} and define the combined optimality variable
\begin{equation}
\mathcal{O}_t := \mathcal{U}_t \wedge \mathcal{V}_t ,
\end{equation}
so that $\mathcal{O}_t=1$ indicates both Lyapunov conditions hold at time $t$.  
The inference distribution over trajectories that satisfy these conditions is
For brevity, we drop the superscript $^{\textnormal{test}}$ on 
test-time variables in the following.
\begin{align}\label{eq:in_context_predictor_uvc}
& p(\mathcal{O}_{\textnormal{traj}}\!=\!1 \mid \mathbf{p}_{1:K}, \mathbf{s}_1) \nonumber \\
&= \int_{\theta} \sum_{z_1 \in \mathcal{Z}} 
   g_{\pi_\theta}(\tau \mid z_1, s_1) \, 
   p(z_1 \mid \mathbf{p}_{1:K}, s_1, \theta) \nonumber\\
&\quad \times \prod_{t=1}^{T} p(\mathcal{O}_t \mid s_t, a_t) \,
   e^{K r_K(\theta)} \, p(\theta) \, d\theta.
\end{align}


where the rollout likelihood under $\pi_\theta$ is defined as
\begin{equation*}
\begin{aligned}
&g_{\pi_\theta}(\tau \mid z_1^{\textnormal{test}}, s_1^{\textnormal{test}})
= \prod_{t=1}^{T} \Big[ \hat{\mathcal{P}}(s_{t+1}^{\textnormal{test}} \mid s_t^{\textnormal{test}}, a_t^{\textnormal{test}}) \\
&\;\;\times \pi_\phi^{\textnormal{low}}(a_t^{\textnormal{test}} \mid s_t^{\textnormal{test}}, z_t^{\textnormal{test}})
\, \pi_\theta^{\textnormal{high}}(z_t^{\textnormal{test}} \mid s_t^{\textnormal{test}}, z_{t-1}^{\textnormal{test}}) \Big]
\end{aligned}
\end{equation*}
(with $s_1^{\textnormal{test}}$ suppressed in subsequent uses for brevity), and
\begin{equation}
r_K(\theta) = \tfrac{1}{K}\log
\frac{p(\mathbf{p}_{1:K},\; s_1^{\textnormal{test}} \mid \theta)}
     {p(\mathbf{p}_{1:K},\; s_1^{\textnormal{test}} \mid \theta^*)}
\end{equation}
measures the compatibility between the prompt $\mathbf{p}_{1:K}$ and $\theta$.  
Here, $\theta^*$ is for the high-level policy from which the prompt $\mathbf{p}_{1:K}$ was generated, 
serving as the safe reference policy in the denominator of $r_K(\theta)$.
A trajectory that satisfies \cref{thm:lyapunov_gm} is equivalent to the one induced by the safe reference policy $\pi^{\textnormal{high}}_{\theta^*}$, which maximizes the likelihood of $\mathcal{O}_{\textnormal{traj}}=1$.

This formulation extends the implicit Bayesian inference view 
\citep{xie2021explanation} in two ways: 
(i) it explicitly incorporates the low-level policy 
$\pi_\phi^{\textnormal{low}}$ to model the action space in RL, and 
(ii) it leverages the predictive transition model 
$\hat{\mathcal{P}}(s_{t+1}\!\mid s_t,a_t)$ to generate imagined trajectories consistent with the environment.
To connect the inference formulation with Lyapunov-based safety, 
we now establish a probabilistic guarantee for the trajectory 
returned by Algorithm~\ref{alg:main}.

\begin{assumption}\label{asm:bounded_diff}
The Lyapunov difference $\lVert G_\textnormal{SAS}(s_t,a_t) - G_\textnormal{SAS}(s_{t+1},a_{t+1})\rVert$ is bounded by a constant $L>0$ for all transitions.
\end{assumption}

\begin{theorem}\label{thm:control_invariant}
Under Assumption~\ref{asm:bounded_diff}, for $\tau_{\textnormal{best}}$ returned by Algorithm~\ref{alg:main},
the probability of leaving $\mathcal{R}_G^\textnormal{SAS}$ is bounded by
\begin{align}\label{eq:control_invariant}
P\!\left[\tau_\textnormal{best} \not\subset 
   \mathcal{R}_G^\textnormal{SAS}\right]
&\;\le\;
\left(\frac{\mathbb{E}_{(s,a)\sim \mathcal{D}}
       \big[-\log \hat{\rho}(s,a)\big]}{c_2}\right)^{NT} \nonumber \\
&\quad+ \exp\!\left(-\tfrac{2M\kappa^2(c_2-c_1)^2}{TL^2}\right).
\end{align}
\end{theorem}

The proof is in \cref{proof:hoeffding}.
Here, $c_1,c_2>0$ are constants defining the control-invariant set and $\kappa$ is a scaling constant from
Hoeffding’s inequality.
As the iterations $N,M \to \infty$, the RHS of
\cref{eq:control_invariant} vanishes, implying that
$\tau_\textnormal{best}$ remains in
$\mathcal{R}_G^\textnormal{SAS}$ with high probability.

\paragraph{Practical interpretation.}
The first term decays exponentially in $N \!\times\! T$: increasing the number of imagined rollouts ($N$) or the trajectory length ($T$) tightens the energy-based feasibility check (condition $\mathcal{U}_t$).
The second term decays exponentially in $M$: more re-imagination trials strengthen the Lyapunov-descent check (condition $\mathcal{V}_t$).
Together, the bound confirms that SAS is not merely a heuristic but admits a formal probabilistic safety guarantee whose tightness improves with moderate computational budget (see ablations in \cref{sec:append_abal}).


\begin{algorithm}[t]
\caption{\textbf{S}elf \textbf{A}lignment for \textbf{S}afety (SAS): Prompt generation for safe test-time adaptation}
\label{alg:main}
\begin{algorithmic}[1]
\Require Pretrained transformer from $\mathcal{D}$; test initial state $\mathbf{s}_1^{\textnormal{test}}$; prompt length $K$
\For{$i=1,2,\dots,N$} \Comment{\textcolor{blue}{loop for condition $\mathcal{U}_t$}}
    \State Imagine trajectory $\tau^{(i)} \sim p(\tau \mid \mathbf{s}_1^{\textnormal{test}})$
    \State Compute $E_t$ for all $t$ in $\tau^{(i)}$
    \State $\hat E_i,\ t_i \gets \max_t E_t,\ \arg\max_t E_t$
\EndFor
\State $i^* \gets \arg\min_i \hat E_i$ 
\State $\hat{\mathbf{p}}_{1:K} \gets \tau^{(i^*)}[t_{i^*}{-}K{+}1: t_{i^*}]$ 
\Statex
\For{$j=1,2,\dots,M$} \Comment{\textcolor{blue}{loop for condition $\mathcal{V}_t$}}
    \State Imagine trajectory $\tau^{(j)} \sim p(\tau \mid \hat{\mathbf{p}}_{1:K},\, \mathbf{s}_1^{\textnormal{test}})$
    \State Compute $E_t$ for all $t$ in $\tau^{(j)}$
    \State $t_j \gets \arg\max_t E_t$
    \State $v_j \gets \sum_t \mathcal{V}_t$ 
\EndFor
\State $j^* \gets \arg\max_j v_j$ 
\State $\mathbf{p}_{1:K} \gets \tau^{(j^*)}[t_{j^*}{-}K{+}1 : t_{j^*}]$ 
\State \Return $\mathbf{p}_{1:K}$ and deploy $\pi_{\theta^*}^{\mathrm{high}}$ by conditioning on $(\mathbf{p}_{1:K}, \mathbf{s}_1^{\textnormal{test}})$ 
\end{algorithmic}
\end{algorithm}

\begin{figure*}[t]
    \centering
    \includegraphics[width=0.8\linewidth]{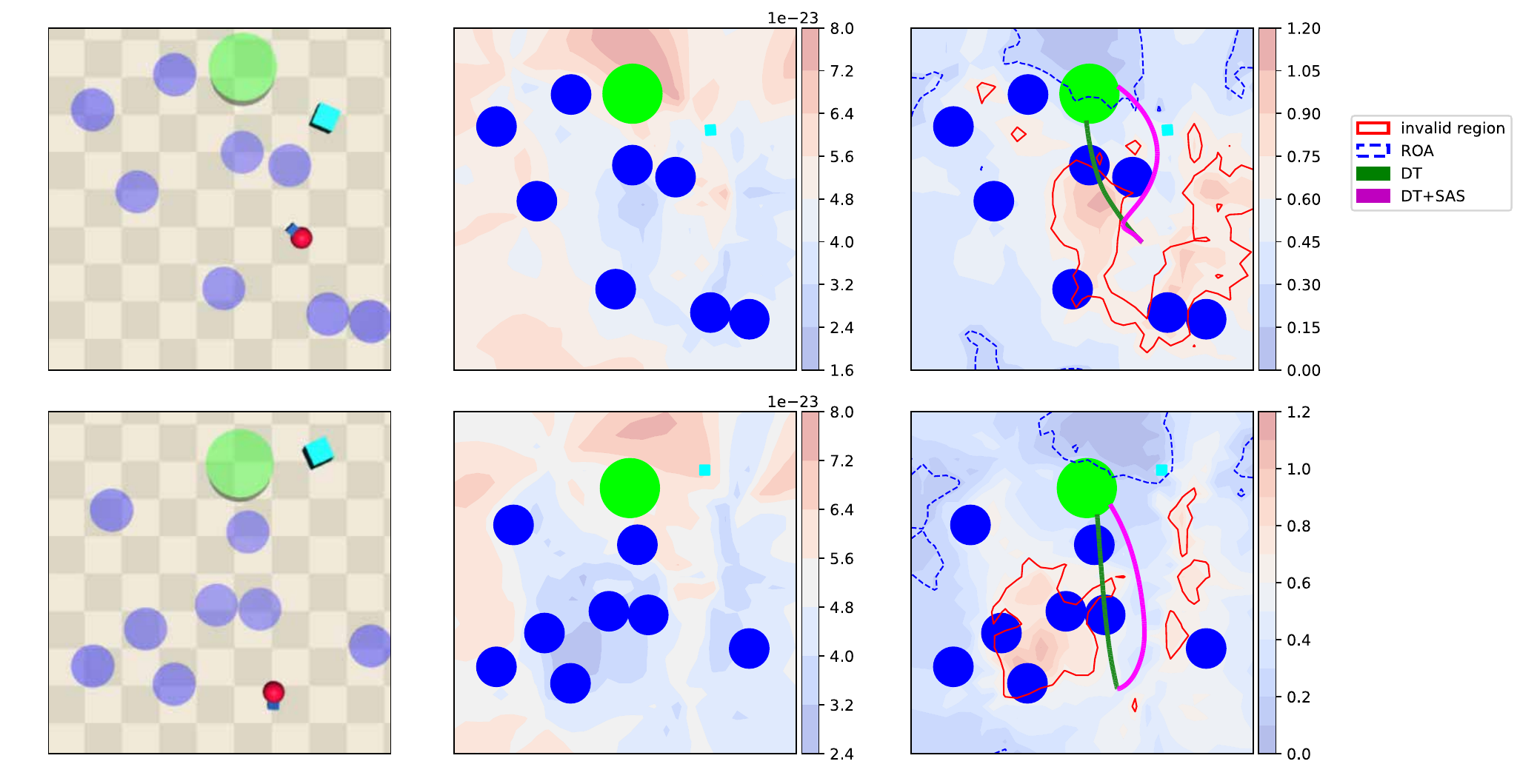}
    \vspace{-0.1in}
    \caption{\textbf{SAS dodges hazard better.}
\textbf{Left:} env. with fixed \textbf{\textcolor{blue}{hazards}}, a movable \textbf{\textcolor{cyan}{obstacle}}, and a \textbf{\textcolor{green}{goal}}. 
\textbf{Middle:} The state density from action-averaged occupancy measure.
\textbf{Right:} Trajectories with and without SAS, overlaid on $G_\textnormal{SAS}$ landscape. 
Blue regions denote control-invariant set $\mathcal{R}^\text{SAS}_G$, red lines indicate unsafe areas.
}
    \label{fig:vis_pointgoal}
\vspace{-0.2in}
\end{figure*}

\subsection{Safe-Alignment for Safety (SAS)}

Algorithm~\ref{alg:main} shows our Lyapunov-guided self-alignment 
procedure.
First, from a fixed test initial state, we construct an initial prompt 
$\hat{\mathbf{p}}_{1:K}$ using a conservative min--max rule on the energy 
$E(s,a)=-\log \hat\rho^{\pi}(s,a)$ for condition $\mathcal{U}_t$. 
Second, conditioned on this prompt, we re-imagine trajectories and select the one 
maximizing the cumulation $\sum_t \mathcal{V}_t$. 
The resulting $\mathbf{p}_{1:K}$, combined with the initial state, defines safe test-time behavior.

In \cref{alg:main}, lines~1--6 choose $i^{*}=\arg\min_i \max_t E_t$ among rollouts 
from $\mathbf{s}_1^{\textnormal{test}}$ and set the initial prompt as the last $K$ pairs at $t_{i^{*}}$ (line~7).
Lines~9--13 then evaluate $v_j=\sum_t \mathcal{V}_t$ for each conditioned rollout. 
Lines~14--15 pick $j^{*}=\arg\max_j v_j$ and form the final prompt from the $K$ pairs at $t_{j^{*}}$.
Finally, line~16 conditions the transformer on 
$(\mathbf{p}_{1:K}, \mathbf{s}_1^{\textnormal{test}})$ to yield the safe policy.

\section{EXPERIMENTS}\label{sec:exp}

In this section, we empirically evaluate the performance of SAS,
comparing it to vanilla offline RL, prior offline safe RL, and other test-time adaptation methods in the \texttt{MuJoCo} \citep{1606.01540}, \texttt{Safety Gymnasium} \citep{ji2023safetygymnasium}, and \texttt{Bullet-Safety-Gym} \citep{Gronauer2022BulletSafetyGym} benchmarks.
We also present analyses and ablation on the policy extraction under self-alignment.

\subsection{Experimental Setup}
We employ the D4RL dataset \citep{fu2021d4rl} for \texttt{MuJoCo} and the DSRL dataset \citep{liu2023datasets} for \texttt{Safety Gymnasium}. Both reward and cost returns are normalized for comparability across tasks.
The methods denoted by \textit{SAS} correspond to our self-alignment method. 
Throughout, \emph{‘+SAS’} denotes our Lyapunov-guided self-alignment applied to a pretrained backbone (DT or CDT) that includes our VAE-augmented world model for imagination. We do not report a standalone \emph{‘SAS’} without a backbone; all our results are backbone+SAS.
For brevity, task names are abbreviated, e.g., \texttt{PointGoal1} (PG1), \texttt{PointPush1} (PP1), and \texttt{CarButton2} (CB2). Detailed experimental settings are provided in \cref{appendix:intro} due to space limitations.

\subsection{Results and Discussions}
We present our experimental results and corresponding insights in a Q\&A style. 
Although referenced in the discussion for clarification, figures and tables not included in the main text are provided in \cref{sec:complete_exp_results}.

\begin{table*}[!tb]
\caption{Full results in \texttt{Safety Gymnasium}. Values are averaged over three cost thresholds, 20 evaluation episodes, and three random seeds. \tgr{Gray}: unsafe agents. \textbf{Bold}: safe agents with normalized cost less than 1. \tb{\textbf{Blue}}: safe agents achieving the highest reward.}
\label{tab:cdt}
\vspace{-0.1in}

\centering

   \resizebox{\textwidth}{!}{%
{\renewcommand{\arraystretch}{1.2}

\begin{tabular}{|cl|cc|cc|cc|cc|cc|cc|cc|cc|cc|cc|}
\hline
\multicolumn{2}{|c|}{}                       & \multicolumn{2}{c|}{\textbf{DT + SAS}}                                              & \multicolumn{2}{c|}{\textbf{CDT + SAS}}                                             & \multicolumn{2}{c|}{CDT}                                                   & \multicolumn{2}{c|}{BC-All}                                & \multicolumn{2}{c|}{BC-Safe}                                                & \multicolumn{2}{c|}{BCQ-Lag}                                                & \multicolumn{2}{c|}{BEAR-Lag}                              & \multicolumn{2}{c|}{CPQ}                                   & \multicolumn{2}{c|}{COptiDICE}     & \multicolumn{2}{c|}{DCRL}                        \\ \cline{3-22} 
\multicolumn{2}{|c|}{\multirow{-2}{*}{Task}} & reward                               & cost                                 & reward                               & cost                                 & reward                               & cost                                & reward                       & cost                        & reward                               & cost                                 & reward                               & cost                                 & reward                       & cost                        & reward                      & cost                         & reward                       & cost             & reward                       & cost           \\ \hline
\multicolumn{2}{|c|}{PointGoal1}             & \tgr{0.66}          & \tgr{1.19}          & \tgr{0.65}          & \tgr{1.27}          & \tgr{0.69}          & \tgr{1.12}         & \textbf{0.65}                & \textbf{0.95}               & \textbf{0.43}                        & \textbf{0.54}                        & \tb{\textbf{0.71}} & \tb{\textbf{0.98}} & \tgr{0.74}  & \tgr{1.18} & \textbf{0.57}               & \textbf{0.35}                & \tgr{0.49}  & \tgr{1.66}        & \textbf{0.24}  & \textbf{0.86}\\
\multicolumn{2}{|c|}{PointGoal2}             & \tgr{0.65}          & \tgr{1.78}          & \tb{\textbf{0.52}} & \tb{\textbf{0.94}} & \tgr{0.59}          & \tgr{1.34}         & \tgr{0.54}  & \tgr{1.97} & \textbf{0.29}                        & \textbf{0.78}                        & \tgr{0.67}          & \tgr{3.18}          & \tgr{0.67}  & \tgr{3.11} & \tgr{0.4}  & \tgr{1.31}  & \tgr{0.38}  & \tgr{1.92} & \textbf{0.28}  & \textbf{0.26}\\
\multicolumn{2}{|c|}{PointPush1}             & \textbf{0.28}                        & \textbf{0.62}                        & \textbf{0.26}                        & \textbf{0.54}                        & \textbf{0.24}                        & \textbf{0.48}                       & \textbf{0.19}                & \textbf{0.61}               & \textbf{0.13}                        & \textbf{0.43}                        & \tb{\textbf{0.33}} & \tb{\textbf{0.86}} & \textbf{0.22}                & \textbf{0.79}               & \textbf{0.2}                & \textbf{0.83}                & \textbf{0.13}                & \textbf{0.83}      & \textbf{0.01}  & \textbf{0.52}         \\
\multicolumn{2}{|c|}{PointPush2}             & \tb{\textbf{0.24}} & \tb{\textbf{0.64}} & \textbf{0.20}                        & \textbf{0.53}                        & \textbf{0.21}                        & \textbf{0.65}                       & \textbf{0.18}                & \textbf{0.91}               & \textbf{0.11}                        & \textbf{0.8}                         & \textbf{0.23}                        & \textbf{0.99}                        & \tgr{0.16}  & \tgr{0.89} & \tgr{0.11} & \tgr{1.04}  & \tgr{0.02}  & \tgr{1.18} & \textbf{0.02}  & \textbf{0.07}\\
\multicolumn{2}{|c|}{PointButton1}           & \tgr{0.49}          & \tgr{1.38}          & \tgr{0.51}          & \tgr{1.27}          & \tgr{0.5}           & \tgr{1.68}         & \tgr{0.1}   & \tgr{10.5} & \tb{\textbf{0.06}} & \tb{\textbf{0.52}} & \tgr{0.24}          & \tgr{1.73}          & \tgr{0.2}   & \tgr{1.6}  & \tgr{0.69} & \tgr{3.2}   & \tgr{0.13}  & \tgr{1.4}  & \textbf{0.01}  & \textbf{0.48}\\
\multicolumn{2}{|c|}{PointButton2}           & \tgr{0.51}          & \tgr{1.14}          & \tb{\textbf{0.41}} & \tb{\textbf{0.98}} & \tgr{0.46}          & \tgr{1.57}         & \tgr{0.27}  & \tgr{2.02} & \tgr{0.16}          & \tgr{1.1}           & \tgr{0.4}           & \tgr{2.66}          & \tgr{0.43}  & \tgr{2.47} & \tgr{0.58} & \tgr{4.3}   & \tgr{0.15}  & \tgr{1.51} & \textbf{0.18}  & \textbf{0.64}\\
\multicolumn{2}{|c|}{CarGoal1}               & \tb{\textbf{0.67}} & \tb{\textbf{0.85}} & \textbf{0.65}                        & \textbf{0.90}                        & \tgr{0.66}          & \tgr{1.21}         & \textbf{0.39}                & \textbf{0.33}               & \textbf{0.24}                        & \textbf{0.28}                        & \textbf{0.47}                        & \textbf{0.78}                        & \tgr{0.61}  & \tgr{1.13} & \tgr{0.79} & \tgr{1.42}  & \textbf{0.35}                & \textbf{0.54}     & \textbf{0.35}  & \textbf{0.88}          \\
\multicolumn{2}{|c|}{CarGoal2}               & \tgr{0.48}          & \tgr{1.15}          & \tb{\textbf{0.42}} & \tb{\textbf{0.98}} & \tgr{0.48}          & \tgr{1.25}         & \tgr{0.23}  & \tgr{1.05} & \textbf{0.14}                        & \textbf{0.51}                        & \tgr{0.3}           & \tgr{1.44}          & \tgr{0.28}  & \tgr{1.01} & \tgr{0.65} & \tgr{3.75}  & \textbf{0.25}                & \textbf{0.91}       & \tgr{0.11}  & \tgr{2.51}        \\
\multicolumn{2}{|c|}{CarPush1}               & \tb{\textbf{0.31}} & \tb{\textbf{0.51}} & \tb{\textbf{0.31}} & \tb{\textbf{0.49}} & \tb{\textbf{0.31}} & \tb{\textbf{0.4}} & \textbf{0.22}                & \textbf{0.36}               & \textbf{0.14}                        & \textbf{0.33}                        & \textbf{0.23}                        & \textbf{0.43}                        & \textbf{0.21}                & \textbf{0.54}               & \textbf{-0.03}              & \textbf{0.95}                & \textbf{0.23}                & \textbf{0.5}        & \textbf{-0.1}  & \textbf{0.09}        \\
\multicolumn{2}{|c|}{CarPush2}               & \tgr{0.22}          & \tgr{1.16}          & \tb{\textbf{0.21}} & \tb{\textbf{0.75}} & \tgr{0.19}          & \tgr{1.3}          & \textbf{0.14}                & \textbf{0.9}                & \textbf{0.05}                        & \textbf{0.45}                        & \tgr{0.15}          & \tgr{1.38}          & \tgr{0.1}   & \tgr{1.2}  & \tgr{0.24} & \tgr{4.25}  & \tgr{0.09}  & \tgr{1.07} & \textbf{-0.13}  & \textbf{0.17}\\
\multicolumn{2}{|c|}{CarButton1}             & \tgr{0.17}          & \tgr{1.08}          & \tb{\textbf{0.27}} & \tb{\textbf{0.98}} & \tgr{0.21}          & \tgr{1.6}          & \tgr{0.03}  & \tgr{1.38} & \textbf{0.07}                        & \textbf{0.85}                        & \tgr{0.04}          & \tgr{1.63}          & \tgr{0.18}  & \tgr{2.72} & \tgr{0.42} & \tgr{9.66}  & \tgr{-0.08} & \tgr{1.68}& \textbf{0.12}  & \textbf{0.95} \\
\multicolumn{2}{|c|}{CarButton2}             & \tb{\textbf{0.14}} & \tb{\textbf{0.84}} & \tgr{0.30}          & \tgr{1.11}          & \tgr{0.13}          & \tgr{1.58}         & \tgr{-0.13} & \tgr{1.24} & \textbf{-0.01}                       & \textbf{0.63}                        & \tgr{0.06}          & \tgr{2.13}          & \tgr{-0.01} & \tgr{2.29} & \tgr{0.37} & \tgr{12.51} & \tgr{-0.07} & \tgr{1.59} & \tgr{0.09}  & \tgr{1.42}\\ \hline
\multicolumn{2}{|c|}{Average}               
& 0.402 & 1.03    
& 0.392 & \textbf{0.895}  
& 0.389 & 1.18    
& 0.234 & 1.85    
& 0.151 & \textbf{0.602}  
& 0.319 & 1.52    
& 0.316 & 1.58    
& 0.416 & 3.63    
& 0.172 & 1.23    
& 0.098 & \textbf{0.737}  
\\ \hline
\end{tabular}%
}
}
\vspace{-0.1in}
\end{table*}

\paragraph{Q1: Why is SAS safer at test time?}\mbox{}\\

\cref{fig:vis_pointgoal} illustrates how SAS guides safer behavior in the \texttt{PointGoal1-v0} task from \texttt{Safety Gymnasium} benchmark~\citep{ji2023safetygymnasium}.
As shown in the left column, the agent (\textcolor{red}{\CIRCLE}) is initialized at a random position with a random orientation. 
In the right column, the resulting trajectories illustrate that \textcolor[rgb]{1,0,1}{DT+SAS}, unlike \textcolor[rgb]{0,0.5,0}{DT}, reaches the goal (\textcolor{green}{\CIRCLE}) without entering any hazard (\textcolor{blue}{\CIRCLE}), thus ensuring a safe trajectory toward the target.

We define the \emph{action-averaged occupancy measure} as the average of $\hat{\rho}^\pi(s,a)$ over actions $\pm x, \pm y$ at each state, since actions represent the intended movement direction of the agent. 
Similarly, $G_\textnormal{SAS}(\mathbf{s}_t, \mathbf{a}_t)$ is averaged in the same manner. 
The \emph{invalid region} corresponds to the area where $G_\textnormal{SAS}$ exceeds its 95th percentile, while the \emph{region of attraction (ROA)} denotes a tight control-invariant set defined by the 10th percentile of $G_\textnormal{SAS}$, which tends to concentrate near the goal (\textcolor{green}{\CIRCLE}).

The safer behavior of \textcolor[rgb]{1,0,1}{DT+SAS} can be explained by comparing the action-averaged occupancy measure in the middle column with $G_\textnormal{SAS}$ in the right column. 
While $G_\textnormal{SAS}$ assigns high values to directions leading toward hazard-dense areas (\textcolor{blue}{\CIRCLE}), thereby forming invalid regions, the action-averaged occupancy measure yields more blurry boundaries. 
Consequently, \textcolor[rgb]{0,0.5,0}{DT}, having limited foresight, often enters such regions, ultimately incurring costs by stepping into hazards.

Although neither \textcolor[rgb]{0,0.5,0}{DT} nor \textcolor[rgb]{1,0,1}{DT+SAS} explicitly uses cost, both are trained on expert data from the offline safe RL dataset DSRL, where agents learn to avoid hazards and reach the goal. 
However, the occupancy measure $\hat{\rho}^\pi$ of \textcolor[rgb]{0,0.5,0}{DT} exhibits myopic behavior in long-horizon evaluation, deeming current actions as probable without sufficient consideration of future risks. 
By contrast, Lyapunov-guided self-alignment in SAS refines the policy at test time, enabling deployment of an agent that is both safer and more aligned with the offline distribution.

\paragraph{Q2: Does SAS consistently improve?}\mbox{}\\

\begin{center}
\captionsetup{type=table} 
\captionof{table}{Performance of DT and DT+SAS in \texttt{MuJoCo}.}
\label{tab:mujoco}
 \resizebox{\linewidth}{!}{%
\begin{tabular}{|cc|cc|cc|cc|}
\hline
\multicolumn{2}{|c|}{}                                    & \multicolumn{2}{c|}{Hopper}                   & \multicolumn{2}{c|}{Walker2d}                               & \multicolumn{2}{c|}{Humanoid}  \\ \cline{3-8} 
\multicolumn{2}{|c|}{\multirow{-2}{*}{Environment}}       & expert         & medium                       & expert                       & medium                       & expert        & medium         \\ \hline
\multicolumn{1}{|c|}{}                          & reward  & 110.7          & 86.6 & 107.7 & 82.2 & 98.5 & 40.5          \\
\multicolumn{1}{|c|}{\multirow{-2}{*}{DT}}      & failure & 0.05           & 1     & 0     & 0.54 & 0.20  & 0.97          \\ \hline
\multicolumn{1}{|c|}{}                          & reward  & 110.7          & \textbf{87.5}                & 107.7 & \textbf{89.5}               & \textbf{103.5}          & \textbf{50.6} \\
\multicolumn{1}{|c|}{\multirow{-2}{*}{\textbf{DT+SAS}}} & failure & \textbf{0.03} & 1     & 0     & \textbf{0.46}               & \textbf{0.10}          & \textbf{0.87} \\ \hline
\end{tabular}%
}
\end{center}

First, \cref{tab:mujoco} and \cref{tab:cql_ape-v}, whose rewards are normalized following \citep{kumar2020conservative}, demonstrate that SAS yields notable 
improvements in both reward (return) and failure rate (safety) even in general 
offline RL without explicit cost annotations, such as D4RL. 
Here, the failure rate is defined as the \emph{early termination} event in 
\texttt{MuJoCo}, where an episode ends prematurely due to exceeding internal 
unhealthy cost thresholds before reaching the maximum episode length.

As shown in \cref{tab:sas_modified}, SAS achieves substantially lower costs in 
\texttt{Safety Gymnasium} during test-time adaptation. 
Note that DT does not leverage cost information in its offline pretraining dataset, 
while CDT incorporates cost signals as in the original pretraining. 
Surprisingly, in most environments, both backbones become safer once our density-based 
Lyapunov guidance is applied. This observation is analogous to the results of LDM, 
yet our method attains safety improvements without requiring a new 
from-scratch pretraining, offering a simple and effective enhancement.

Finally, \cref{tab:cdt} and \cref{tab:cdt_Appendix} report aggregated benchmark results in safe RL. 
\cref{tab:cdt_Appendix} extends \cref{tab:cdt} by including 8 additional tasks from \texttt{bullet-safety-gym} on top of the 16 tasks in \texttt{Safety Gymnasium}. 
Across heterogeneous agents with varying observations and dynamics (\texttt{point, car, ant, ball, drone}), SAS consistently outperforms, including strong safe-control methods such as BC-Safe, and yields superior results over its own backbones DT and CDT. 
Among the methods that satisfy the safety criterion (average normalized cost $\leq 1$), namely CDT+SAS, BC-Safe, and DCRL, CDT+SAS achieves the highest average reward.
Per-task standard deviations across seeds and cost thresholds are reported in \cref{tab:sas_modified}; SAS consistently reduces variance in cost relative to its backbone.

This effectiveness without explicit cost signals can be attributed to the structure of the offline safe RL datasets (e.g., DSRL), where expert demonstrations naturally concentrate in safe regions.
Consequently, high-density states under the occupancy measure $\hat{\rho}^\pi$ largely coincide with low-cost states, allowing the density-based Lyapunov model to serve as an implicit safety proxy.

\paragraph{Q3: Is SAS better than random or max prompt?}\mbox{}\\
In \cref{tab:sas}, DT+SAS exhibits consistently lower costs than the random prompting baseline (DT+rand), whose costs often exceed those of DT in half of \texttt{Safety Gymnasium}. To validate the role of Lyapunov-guided selection, we also consider an alternative ablation (DT+maxmax) that chooses the trajectory with the largest maximum step-wise energy. This variant yields higher costs and failure rates, and generally lower rewards, compared to DT+SAS. These results confirm that SAS provides a valid and effective self-generated prompt: by selecting trajectories that minimize the worst-case energy, the transformer makes safer decisions throughout the episode, analogous to choosing an \emph{appropriate high-level skill in hierarchical RL}.


\begin{figure*}[t]
    \centering
    \includegraphics[width=0.95\linewidth]{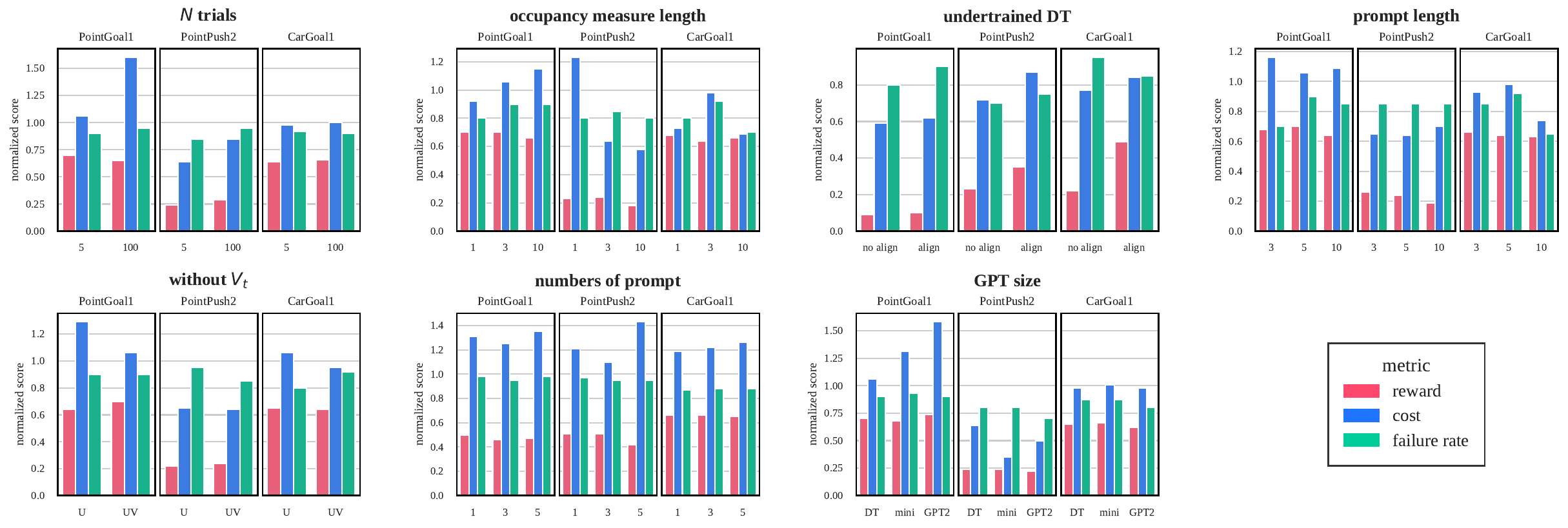}
\caption{\textbf{Comprehensive ablation of SAS.} 
Normalized \tr{reward}, \tb{cost}, and \tg{failure rate} on 
\texttt{PointGoal1}/\texttt{PointPush2}/\texttt{CarGoal1}. 
\textbf{Top:} number of imagined trajectories per loop ($N\!\in\!\{5,100\}$); 
horizon used to aggregate energy $E\!=\!-\log\hat\rho$ ($\{1,3,10\}$); 
self-alignment on an under-pretrained DT (\textit{no align} vs.\ \textit{align}); 
prompt length $K\!\in\!\{3,5,10\}$. 
\textbf{Bottom:} ablation of the Lyapunov-descent condition 
($\mathcal{U}_t$-only [U] vs.\ $\mathcal{U}_t\!\wedge\!\mathcal{V}_t$ [UV]); 
number of in-context prompt fragments concatenated ($\{1,3,5\}$); 
backbone scale (DT, GPT-mini, GPT-2).}
    \label{fig:hist_full}
\end{figure*}

\paragraph{Q4: Does SAS benefit undertrained models?}\mbox{}\\
The \emph{undertrained DT} panel of \cref{fig:hist_full} compares an under-pretrained 
DT backbone with (\textit{align}) and without (\textit{no align}) self-alignment. 
SAS improves the reward across all three environments, even though cost and 
failure rate also rise relative to the well-trained baseline. 
It implies that SAS can extract useful behavior from partially trained models at test 
time -- the Lyapunov-guided prompt steers the agent toward in-distribution.

\paragraph{Q5: How essential is $\mathcal{V}_t$?}\mbox{}\\
The \emph{without $\mathcal{V}_t$} panel of \cref{fig:hist_full} ablates the 
Lyapunov-descent stage by comparing $\mathcal{U}_t$ alone (U) with the full 
$\mathcal{U}_t \wedge \mathcal{V}_t$ (UV). 
Adding $\mathcal{V}_t$ consistently lowers cost across all three environments; 
in terms of failure rate, the improvement carries over to \texttt{PointGoal1} 
and \texttt{PointPush2}, with only a slight uptick on \texttt{CarGoal1}. 
This confirms that $\mathcal{V}_t$ is the component stabilizing long-horizon 
rollouts -- structurally, it realizes the second term of 
\cref{eq:control_invariant} governed by $M$ and $\kappa$, so removing it 
($\kappa{=}0$) tolerates a violation at every step and inflates cost.

\paragraph{Q6: Is SAS robust across design choices?}\mbox{}\\
The remaining five panels of \cref{fig:hist_full} sweep the main hyperparameters 
of SAS: number of imagined trajectories $N\!\in\!\{5,100\}$, the energy 
aggregation horizon $\{1,3,10\}$ (\emph{occupancy measure length}), prompt length 
$K\!\in\!\{3,5,10\}$, the number of prompt fragments concatenated 
$\{1,3,5\}$ (\emph{numbers of prompt}), and backbone scale (DT $\to$ GPT-mini $\to$ GPT-2). 
Variations within each axis are small and environment-specific, confirming that 
SAS is broadly robust to these choices. 
Two of these axes have a direct interpretation through \cref{thm:control_invariant}: 
$N$ tightens the first ($\mathcal{U}_t$) term of the bound, although excessively 
large $N$ (e.g., $100$) can overfit to extremely high-occupancy regions and erode 
task return; the energy aggregation horizon corresponds to $T$, where 
$G_{\textnormal{SAS}}$ captures the worst-case energy along a trajectory, so a 
moderate horizon already pinpoints the decisive unsafe region. 
The remaining axes (prompt length, number of fragments, backbone size) are not 
directly tied to the bound but show consistent saturation, in line with the 
qualitative robustness predicted by the analysis. 
Per-axis details are deferred to \cref{sec:append_abal}.

\paragraph{Related work.}
SAS sits at the intersection of three lines.
\emph{(i) Transformer-based RL}---Decision Transformer
\citep{chen2021decision,janner2021sequence}, transformer world models
\citep{robine2023transformer,micheli2023transformers}, and
prompt-/cost-conditioning extensions
\citep{pmlr-v202-jiang23b,xu2022prompting,liu2023constrained}---provides
the autoregressive imagination backbone we adapt, but does not address
safety alignment at test time.
\emph{(ii) Safe RL with Lyapunov-style certificates}
\citep{chow2018lyapunov,chang2019neural,kang2022lyapunov,qin2021density,zheng2024safe}
acts at \emph{training time}, learning a constrained or feasibility-aware
policy from scratch.
\emph{(iii) Alignment of pretrained models} is dominated by RLHF
\citep{ouyang2022training} and prompt-based instruction tuning
\citep{askell2021general,sun2023principledriven}, but safety alignment of an
offline RL agent \emph{at test time} remains largely unexplored.
SAS fills this gap by aligning an offline-pretrained transformer at deployment
through Lyapunov-guided self-generated prompts, without parameter updates or
human supervision.
A full discussion is deferred to \cref{sec:related} in the appendix.

\section{CONCLUSION}

We presented SAS, a Lyapunov-guided test-time adaptation for offline RL that improves safety without retraining, consistently lowering cost and failure while maintaining competitive returns.
Current limitations include the additional inference cost from multiple imagined rollouts and the reliance on data density rather than explicit cost constraints, making safety guarantees dependent on offline dataset coverage.
We hope SAS encourages future work on reducing inference overhead, incorporating explicit constraint signals, and broader test-time adaptation for safety.

\subsubsection*{Acknowledgments}
This work was supported in part by the National Research Foundation of Korea (NRF) 
(grant numbers RS-2024-00451435 (20\%) and RS-2024-00413957 (20\%)); 
the Institute of Information \& Communications Technology Planning \& Evaluation (IITP) 
(grant numbers RS-2025-02305453 (15\%), RS-2025-02273157 (15\%), 
RS-2025-25442149 (15\%), and RS-2021-II211343 (15\%)) 
funded by the Ministry of Science and ICT (MSIT); 
the Institute of New Media and Communications (INMAC); 
the BK21 FOUR Program of the Education and Artificial Intelligence Graduate School Program 
at Seoul National University; 
and the Research Program for Future ICT Pioneers at Seoul National University (2026).

\bibliography{src/ref}

\section*{Checklist}

\begin{enumerate}

  \item For all models and algorithms presented, check if you include:
  \begin{enumerate}
    \item A clear description of the mathematical setting, assumptions, algorithm, and/or model. [Yes]  
    Justification: Section 2 (Background) and Section 3 (Control Invariant Set, Lyapunov Model) clearly define the MDP setting, assumptions, and the proposed SAS algorithm.
    
    \item An analysis of the properties and complexity (time, space, sample size) of any algorithm. [No]  
    Justification: The paper does not provide explicit computational complexity or sample complexity analysis, although iterative algorithm properties are discussed in Section 4 and Appendix F.
    
    \item (Optional) Anonymized source code, with specification of all dependencies, including external libraries. [Yes]  
    Justification: We provide anonymized source code with dependency specifications in the supplementary material.

  \end{enumerate}

  \item For any theoretical claim, check if you include:
  \begin{enumerate}
    \item Statements of the full set of assumptions of all theoretical results. [Yes]  
    Justification: Assumptions are explicitly stated in the main text (Assumption~1 in Section 4.2) and in Appendix E.
    
    \item Complete proofs of all theoretical results. [Yes]  
    Justification: Full proofs are provided in Appendix F (e.g., Proof of Theorem 2, Theorem 3, and Eq. (8)).
    
    \item Clear explanations of any assumptions. [Yes]  
    Justification: Explanations of assumptions (e.g., bounded Lyapunov difference) are given inline in Section 4.2 and detailed in Appendix E.
  \end{enumerate}

  \item For all figures and tables that present empirical results, check if you include:
  \begin{enumerate}
    \item The code, data, and instructions needed to reproduce the main experimental results (either in the supplemental material or as a URL). [Yes]  
    Justification: We provide anonymized source code and reproduction instructions in the supplementary material; datasets (D4RL, DSRL, Safety Gymnasium) are publicly available and cited.

    \item All the training details (e.g., data splits, hyperparameters, how they were chosen). [Yes]  
    Justification: Section C (Appendix) provides experiment setting, dataset details, normalization, and hyperparameters (Table 5).
    
    \item A clear definition of the specific measure or statistics and error bars (e.g., with respect to the random seed after running experiments multiple times). [Yes]  
    Justification: Tables 1–3 report averages over 20 evaluation episodes and 3 random seeds; normalized metrics are defined in Appendix C.2.
    
    \item A description of the computing infrastructure used. (e.g., type of GPUs, internal cluster, or cloud provider). [No]  
    Justification: The paper does not specify the hardware or compute infrastructure.
  \end{enumerate}

  \item If you are using existing assets (e.g., code, data, models) or curating/releasing new assets, check if you include:
  \begin{enumerate}
    \item Citations of the creator If your work uses existing assets. [Yes]  
    Justification: The paper cites D4RL, RL-Unplugged, DSRL, and Safety Gymnasium.
    
    \item The license information of the assets, if applicable. [No]  
    Justification: Licenses for datasets are not specified.
    
    \item New assets either in the supplemental material or as a URL, if applicable. [Not Applicable]  
    Justification: No new dataset or model asset is released.
    
    \item Information about consent from data providers/curators. [Not Applicable]  
    Justification: Only publicly available benchmark datasets are used.
    
    \item Discussion of sensible content if applicable, e.g., personally identifiable information or offensive content. [Not Applicable]  
    Justification: The datasets contain only robotics and control trajectories, no sensitive content.
  \end{enumerate}

  \item If you used crowdsourcing or conducted research with human subjects, check if you include:
  \begin{enumerate}
    \item The full text of instructions given to participants and screenshots. [Not Applicable]  
    \item Descriptions of potential participant risks, with links to Institutional Review Board (IRB) approvals if applicable. [Not Applicable]  
    \item The estimated hourly wage paid to participants and the total amount spent on participant compensation. [Not Applicable]  
    Justification: No human subjects or crowdsourcing are involved in this work.
  \end{enumerate}

\end{enumerate}

\clearpage
\appendix
\thispagestyle{empty}
\onecolumn
\aistatstitle{Supplementary Materials: \\
Self-alignment for Offline Safe Reinforcement Learning}


\section{RELATED WORK}\label{sec:related}

\paragraph{Transformer-based RL.}
Transformer-based RL \citep{janner2021sequence,chen2021decision} emerged by bridging the autoregressive pretraining paradigm of GPT \citep{radford2018improving} with offline RL on prior data.
More recently, transformer-based world models such as TWM \citep{robine2023transformer} and IRIS \citep{micheli2023transformers} have achieved sample-efficient online RL by leveraging autoregressive imagination.
Prompting in transformer-based RL has also been explored for task specification with multi-modal prompts \citep{pmlr-v202-jiang23b} and test-time adaptation via learned prompts \citep{xu2022prompting}.
CDT~\citep{liu2023constrained} is closely related to our work, as it extends the Decision Transformer architecture to offline safe RL by conditioning on cost returns.
Unlike these prompting and cost-conditioning approaches, SAS aligns a transformer-based world model via in-context learning by providing self-generated, Lyapunov-guided instructions, without any fine-tuning or cost annotations.

\paragraph{Safe RL and Lyapunov condition.}
Lyapunov conditions have been widely applied to safe control \citep{chang2019neural} and safe RL \citep{chow2018lyapunov}.
LDM \citep{kang2022lyapunov} integrates the Lyapunov condition with offline RL to mitigate distribution shift for safety.
While safe RL is often formulated as a constrained MDP using control-theoretic surrogates---reachability \citep{bansal2021deepreach}, iterative constraint enforcement \citep{ganai2023iterative}, or trust-region constraints \citep{kim2023trust}---to prevent entering unsafe regions, we instead leverage the Lyapunov condition to avoid unsafe regions caused by distribution shift \citep{tedrake2009underactuated,bharadhwaj2020conservative}.
DCRL \citep{qin2021density} is related to our approach in that it constrains state density levels to keep the agent in high-probability states.
FISOR \citep{zheng2024safe} also shares a high-level motivation with our work: both leverage a scalar stability surrogate---Hamilton--Jacobi-based feasible value in FISOR, Lyapunov-based density in SAS---to mitigate safety failures from distribution shift.
However, FISOR is a \emph{training-time} offline safe RL algorithm that learns a new diffusion-based policy from scratch, whereas SAS is a \emph{test-time} alignment method that adapts an already-trained agent via in-context self-generated instructions without retraining.
The two approaches are thus complementary: FISOR produces a safe policy during training, while SAS aligns a pretrained model at test time.

\paragraph{Large model alignment.}
Alignment techniques for LLMs aim to make pretrained models safer and more helpful, e.g., steering a general language assistant toward \textit{helpful, honest, harmless} (HHH) behavior \citep{askell2021general}.
These methods broadly fall into RLHF \citep{ouyang2022training} and instruction-based in-context learning \citep{sun2023principledriven,DBLP:conf/acl/WangKMLSKH23}, the latter engineering prompts \citep{brown2020language} or leveraging zero-shot reasoning \citep{kojima2022large} to elicit desired outputs.
In RL, aligning pretrained models for human preferences is relatively well-studied, but adapting them to unseen task specifications via demonstration \citep{pmlr-v202-jiang23b} or augmented prompts \citep{xu2022prompting} remains limited---even though alignment for safety is essential in real-world deployment.
SAS fills this gap by adapting an offline-pretrained agent at test time through Lyapunov-guided self-generated prompts, requiring no human supervision or parameter updates.

\section{MODEL ARCHITECTURE}
\begin{figure}[hbt!]
    \centering
    \includegraphics[width=0.3\linewidth]{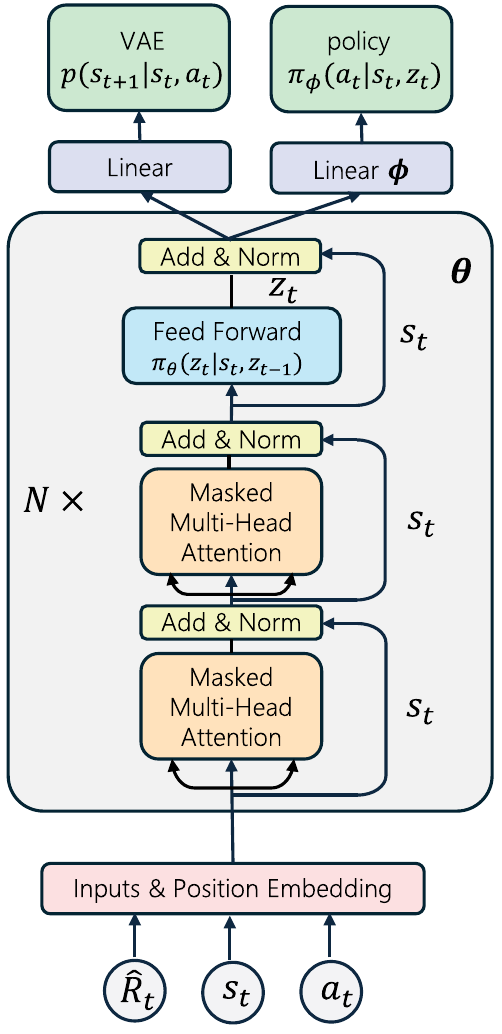}
    \vskip -5pt
    \caption{The architecture of decision transformer with VAE for model-based RL. The only difference with decision transformer is the additional linear layer and VAE decoder to predict the next state. We consider the output of feed-forward layer as the predictor of $z_t$ with the parameter $\theta$ which corresponds to the high-level policy, and the combined values of $s_t, z_t$ by the attention and residual connection are fed into the low-level policy with the linear layer $\phi$. }
    \label{fig:dt}
\end{figure}

\section{EXPERIMENT SETTING AND HYPERPARAMETERS}\label{appendix:intro}

\subsection{Experiment setting}\label{sec:appendix_exp_setting}
We conduct Hopper, Walker2d, and Humanoid in OpenAI Gym, where the agent fails and terminates when the sum of unhealthy rewards get larger.
For Safety Gymnasium, we use two different robots (\texttt{Point, Car}) in 3 tasks (\texttt{Goal, Push, Button}) with two difficulties (\texttt{1,2}) respectively.
In \texttt{Goal} and \texttt{Button} tasks, an agent navigate to the goal while avoiding touching hazards, and an agent push a box to the goal in \texttt{Push} task.
We denote normalized reward and cost returns as reward and cost for simplicity, and use failure in Tables for failure rate.
If an agent experiences any cost due to encountering a hazard within an episode or exceeding unhealthy cost for mujoco (terminated), we considered that episode as a failure episode.
The baselines we used are CDT \citep{liu2023constrained}, Imitation Learning (BC-Safe, BC-All\citep{liu2023constrained, xu2022constraints}), Distribution Correction Estimation (COptiDICE\citep{lee2022coptidice}), and Q-learning (CPQ, BCQ-Lag, BEAR-Lag\citep{xu2022constraints}).

\subsubsection{Baseline Details}
\label{app:baselines}

For completeness, we provide the full set of baseline methods used in our experiments. 
These methods span standard offline RL, imitation learning, distributional correction, 
and Q-learning approaches:

\begin{itemize}
    \item \textbf{Constrained Decision Transformer (CDT)} \citep{liu2023constrained}: 
    A transformer-based offline RL algorithm that incorporates safety constraints 
    into trajectory modeling.

    \item \textbf{Imitation Learning baselines}:  
    \begin{itemize}
        \item \textbf{BC-Safe} \citep{liu2023constrained}: Behavioral cloning restricted to safe demonstrations.  
        \item \textbf{BC-All} \citep{liu2023constrained, xu2022constraints}: Behavioral cloning trained on the entire dataset, 
        including unsafe trajectories.  
    \end{itemize}

    \item \textbf{Distribution Correction Estimation}:  
    \begin{itemize}
        \item \textbf{COptiDICE} \citep{lee2022coptidice}: 
        A distribution correction approach that optimizes cost-sensitive objectives 
        for safe offline RL.  
    \end{itemize}

    \item \textbf{Q-learning baselines}:  
    \begin{itemize}
        \item \textbf{CPQ} \citep{xu2022constraints}: Constrained Policy Q-learning for safe offline RL.  
        \item \textbf{BCQ-Lag} \citep{xu2022constraints}: Offline BCQ augmented with a Lagrangian penalty for safety.  
        \item \textbf{BEAR-Lag} \citep{xu2022constraints}: Offline BEAR similarly extended with Lagrangian safety regularization.  
    \end{itemize}
\end{itemize}

We adopt standard hyperparameters for these baselines from their respective papers. 
All implementations are based on publicly available code where possible.

\subsection{Hyperparameters for the Experiments}
During the training of Decision Transformer, we applied warmup for the first 10000 steps, and we used the ReLU activation function. Further details about the hyperparameters can be found in \cref{fig:hyperparameters}.

\begin{table}[hbt!]
\caption{Hyperparameters for the experiments}
\vskip 0.15in
   \resizebox{\columnwidth}{!}{%
   {\renewcommand{\arraystretch}{1.1}
\begin{tabular}{cccccc}
\cline{1-2} \cline{4-6}
Common Parameters    & Safety-Gymnasium                          &  & Parameters                & CDT          & DT           \\ \cline{1-2} \cline{4-6} 
Action hidden size   & [256, 256] for all methods except CDT, DT &  & Number of layers          & 3            & 3            \\
VAE hidden size      & [400, 400] BCQ-Lag, BEAR-Lag, CPQ         &  & Number of attention heads & 8            & 1            \\
Cost thresholds      & [20, 40, 80]                              &  & Embedding dimension       & 128          & 128          \\
Gradient steps       & 100000                                    &  & Batch size                & 2048         & 64           \\
$[K_\mathcal{P}$, $K_\mathcal{I}$, $K_\mathcal{D}]$         & [0.1, 0.003, 0.001] BCQ-Lag, BEAR-Lag     &  & Context length K          & 300          & 20           \\
Batch size           & 512                                       &  & Learning rate             & 0.0001       & 0.0001       \\
Actor learning rate  & 0.0001                                    &  & Dropout                   & 0.1          & 0.1          \\
Critic learning rate & 0.001                                     &  & Adam betas                & (0.9, 0.999) & (0.9, 0.999) \\ \cline{1-2} \cline{4-6} 
\end{tabular}%
}}
\label{fig:hyperparameters}
\end{table}

\subsection{Normalized Score}

 We applied normalization to both reward return and cost return to make it easier to compare for all environments. Let $r_{max}(\mathcal M)$ and $r_{min}(\mathcal{M})$ denote the maximum reward return and minimum reward return in the dataset $\mathcal{T}$, respectively. Then, the normalized reward return is computed as:
\begin{equation*}
    R_{normalized}=\frac{R_\pi - r_{min}(\mathcal{M})}{r_{max}(\mathcal M)-r_{min}(\mathcal M)}
\end{equation*}
where $\mathcal{R_\pi}$ denotes the evaluated reward return obtained by the agent. Normalized cost return is defined as the ratio between the cost return obtained by the agent and the target cost $\kappa$:
\begin{equation*}
    C_{normalized} = \frac{C_{\pi}+\epsilon}{\kappa + \epsilon}
\end{equation*}
where $\epsilon$ is a small positive number for numerical stability. The values are averaged across three different cost thresholds, 20 evaluation episodes, and three random seeds.

\subsection{Dataset Details}
We conducted experiments using the OpenAI Gym's medium and expert datasets from \url{https://github.com/Farama-Foundation/D4RL} and the Safety Gymnasium's expert dataset from \url{https://github.com/liuzuxin/OSRL/tree/main}. Detailed information about the dataset is presented in \cref{tab:my-tab4}. The Max Cost means the maximum cost return in dataset trajectories.

\begin{table}[hbt!]
\label{tab:my-tab4}
\caption{Dataset details}
\vskip 0.15in
   \resizebox{\columnwidth}{!}{%
   {\renewcommand{\arraystretch}{1.0}
\begin{tabular}{|c|c|c|c|c|c|c|}
\hline
Benchmark                          & Task                  & Max Timestep           & Action Space        & State Space & Max Cost & Trajectories \\ \hline
\multirow{12}{*}{Safety Gymnasium} & SafetyPointGoal1-v0   & \multirow{12}{*}{1000} & \multirow{12}{*}{2} & 60          & 100      & 2022         \\
                                   & SafetyPointGoal2-v0   &                        &                     & 60          & 200      & 3442         \\
                                   & SafetyPointPush1-v0   &                        &                     & 76          & 150      & 2379         \\
                                   & SafetyPointPush2-v0   &                        &                     & 76          & 200      & 3242         \\
                                   & SafetyPointButton1-v0 &                        &                     & 76          & 200      & 2268         \\
                                   & SafetyPointButton2-v0 &                        &                     & 76          & 250      & 3288         \\
                                   & SafetyCarGoal1-v0     &                        &                     & 72          & 200      & 1671         \\
                                   & SafetyCarGoal2-v0     &                        &                     & 72          & 250      & 4105         \\
                                   & SafetyCarPush1-v0     &                        &                     & 88          & 250      & 2871         \\
                                   & SafetyCarPush2-v0     &                        &                     & 88          & 400      & 4407         \\
                                   & SafetyCarButton1-v0   &                        &                     & 88          & 250      & 2656         \\
                                   & SafetyCarButton2-v0   &                        &                     & 88          & 300      & 3755         \\ \hline
\end{tabular}%
}}
\end{table}

\section{COMPARISON WITH FISOR}\label{sec:fisor_comparison}

As discussed in \cref{sec:related}, FISOR \citep{zheng2024safe} is a training-time offline safe RL method, while SAS is a test-time alignment approach.
Although the two methods address different problem settings, we provide a direct empirical comparison for completeness.
\cref{tab:fisor} reports normalized reward and cost on overlapping environments from Safety-Gymnasium and Bullet-Safety-Gym.
CDT+SAS generally achieves higher reward but incurs higher cost, reflecting its performance-oriented design.
FISOR consistently yields lower cost at the expense of some reward, reflecting its safety-oriented training objective.
The two approaches are complementary: FISOR can be used to obtain a safe policy during training, while SAS can further align a pretrained agent at deployment.

\begin{table}[h]
\centering
\caption{Comparison of CDT+SAS and FISOR on overlapping environments. Reward ($\uparrow$) and Cost ($\downarrow$) are normalized. Values below 1.0 for cost indicate safe behavior.}
\label{tab:fisor}
\vspace{0.05in}
\small
\begin{tabular}{lcccc}
\toprule
 & \multicolumn{2}{c}{FISOR} & \multicolumn{2}{c}{CDT+SAS} \\
\cmidrule(lr){2-3}\cmidrule(lr){4-5}
Task & Reward & Cost & Reward & Cost \\
\midrule
\multicolumn{5}{l}{\textit{Safety-Gymnasium}} \\
PointGoal1 & 0.30 & \textbf{0.35} & \textbf{0.65} & 1.27 \\
PointGoal2 & 0.30 & \textbf{0.35} & \textbf{0.52} & 0.94 \\
PointPush1 & \textbf{0.49} & 0.83 & 0.26 & \textbf{0.54} \\
PointPush2 & \textbf{0.44} & \textbf{0.11} & 0.20 & 0.53 \\
PointButton1 & 0.01 & \textbf{0.58} & \textbf{0.51} & 1.27 \\
PointButton2 & 0.01 & \textbf{0.58} & \textbf{0.41} & 0.98 \\
\midrule
\multicolumn{5}{l}{\textit{Bullet-Safety-Gym}} \\
BallRun & \textbf{0.18} & \textbf{0.00} & 0.04 & 0.29 \\
CarRun & \textbf{0.73} & \textbf{0.14} & 0.72 & 0.39 \\
DroneRun & 0.30 & \textbf{0.55} & \textbf{0.38} & 0.78 \\
AntRun & \textbf{0.45} & \textbf{0.03} & 0.29 & 0.45 \\
BallCircle & \textbf{0.34} & \textbf{0.00} & 0.31 & 0.47 \\
CarCircle & \textbf{0.40} & \textbf{0.11} & 0.22 & 0.75 \\
DroneCircle & 0.48 & \textbf{0.00} & \textbf{0.54} & 0.63 \\
AntCircle & 0.20 & \textbf{0.00} & \textbf{0.26} & 0.34 \\
\bottomrule
\end{tabular}
\end{table}

\clearpage
\section{ABLATION STUDIES} \label{sec:append_abal}


\subsection{Number of trajectories sampled for imagination}
We employ Decision Transformer to imagine multiple trajectories under both condition $\mathcal{U}_t$ and condition $\mathcal{V}_t$. In our case, we sampled 5 trajectories for each condition $\mathcal{U}_t$ and $\mathcal{V}_t$. As part of an ablation study, we compared the results of sampling 100 trajectories in experiment with our experimental results. The experiment revealed that there was not a significant difference in the model's performance due to the difference in the number of sampled trajectories. In PointGoal1 environment, an increase in cost was observed when the number of sampled trajectories was 100.

\subsection{Time step length to calculate $E$}
We calculated and approximated $E$ from trajectories imagined by Decision Transformer under both conditions U and V. We conducted experiments with a default time step length of 3 for computing $E$. As part of an ablation study, we also experimented with time step lengths of 1 and 10, comparing the results with our findings. In the results for the CarGoal1 environment, the cost is lowest when the time step length is 10, while in the PointGoal1 environment, it is actually highest. This indicates that increasing the time step length for calculating $E$ does not noticeably improve the model's performance.
\subsection{Time step length of trajectory in prompt}

We extracted the trajectory from a specific time t to 5 time steps before that from trajectories generated through Condition $\mathcal{U}_t$ and $\mathcal{V}_t$. We then fed this truncated trajectory into the prompt of Decision Transformer at the test time. As part of an ablation study, we experimented with the time step length of Decision Transformer's prompt, setting it to 3 and 10. For each time step length of the prompt (3, 5, 10), there are instances where the cost in the experimental results is the highest, as well as instances where it is the lowest. Hence, it can be concluded that the time step length of the prompt does not significantly impact the model's performance.

\subsection{Number of prompts}
We proceeded by using a single trajectory fragment generated by our algorithm as the prompt for Decision Transformer. As part of an ablation study, we compared the performance of our approach with the method of concatenating three or five trajectory fragments obtained by running our algorithm three or five times, respectively, and using them as a prompt. The experimental results show that the method of using five trajectory fragments as a prompt resulted in higher costs. While there is some difference in the PointPush2 environment when the number of fragments is 1 or 3, overall, the performance fluctuates without a clear trend.

\subsection{Model size of Decision Transformer}
We conducted experiments to observe how the effectiveness of SAS varies with the model size of the Decision Transformer. Starting from the smallest size, the default Decision Transformer, we experimented with sizes ranging from GPT-mini to larger sizes like GPT2. When examining the PointGoal1 environment, it seems that as the model size increases, the cost also tends to increase. However, looking at the PointPush2 environment, the opposite trend is observed, where the model with the smallest size has the highest cost, suggesting that there may not be a significant correlation. However, concerning failures, except for the gpt-mini size in PointGoal1, it can be observed that as the model size increases, failures generally decrease.


\section{COMPLETE EXPERIMENT RESULTS}\label{sec:complete_exp_results}

\begin{table*}[hbt!]

\caption{Ablation study in the Safety Gymnasium. DT+rand involves inserting a random trajectory into the prompt, and DT+maxmax includes the trajectory with the argmax of the maximum value of $E$ as the prompt. \textbf{Bold}: the smallest cost among the four models. \tb{\textbf{Blue}}: DT+SAS has a lower failure rate than DT. \tr{\textbf{Red}}: DT+SAS has a higher cost than DT but the reward is also higher.}
\vskip 0.15in
\begin{center}
   \resizebox{0.95\textwidth}{!}{%
   {\renewcommand{\arraystretch}{1.2}
   {\small
\begin{tabular}{|cc|cccccccccccc|}
\hline
\multicolumn{2}{|c|}{Environment}                           & PG1                                   & PG2                                   & PP1                          & PP2                                   & PB1                                   & PB2                                   & CG1                                   & CG2                                   & CP1                                   & CP2                                   & CB1                          & CB2                          \\ \hline
\multicolumn{1}{|c|}{}                            & reward  & \tgr{0.660}          & \tgr{0.377}          & \tgr{0.218} & \tgr{0.202}          & \tgr{0.379}          & \tgr{0.495}          & \tgr{0.638}          & \tgr{0.513}          & \tgr{0.35}           & \tgr{0.204}          & \tgr{0.237} & \tgr{0.212} \\
\multicolumn{1}{|c|}{}                            & cost    & \tgr{1.319}          & \tgr{2.625}          & \tgr{0.927} & \tgr{0.782}          & \textbf{1.188}                        & \tgr{1.309}          & \tgr{0.976}          & \tgr{1.466}          & \tgr{0.678}          & \tgr{1.174}          & \tgr{1.419} & \tgr{1.045} \\
\multicolumn{1}{|c|}{\multirow{-3}{*}{DT}}        & failure & \tgr{0.883}          & \tgr{1.000}          & \tgr{0.667} & \tgr{0.875}          & \tgr{0.950}          & \tgr{0.983}          & \tgr{0.917}          & \tgr{0.925}          & \tgr{0.667}          & \tgr{0.950}          & \tgr{0.950} & \tgr{0.950} \\ \hline
\multicolumn{1}{|c|}{}                            & reward  & \tgr{0.655}          & \tgr{0.650}          & \tgr{0.283} & \tgr{0.242}          & \tr{\textbf{0.485}} & \tgr{0.508} & \tgr{0.666}          & \tgr{0.483}          & \tgr{0.307}          & \tgr{0.218}          & \tgr{0.174} & \tgr{0.138} \\
\multicolumn{1}{|c|}{}                            & cost    & \tgr{1.185}          & \textbf{1.783}                        & \textbf{0.622}               & \textbf{0.639}                        & \tgr{1.375}          & \tgr{1.205}          & \textbf{0.846}                        & \textbf{1.148}                        & \textbf{0.513}                        & \textbf{1.158}                        & \textbf{1.083}               & \textbf{0.836}               \\
\multicolumn{1}{|c|}{\multirow{-3}{*}{\textbf{DT+SAS(ours)}}}   & failure & \tb{\textbf{0.867}} & \tb{\textbf{0.983}} & \tgr{0.767} & \tb{\textbf{0.850}} & \tgr{0.950}          & \tb{\textbf{0.967}} & \tb{\textbf{0.867}} & \tb{\textbf{0.850}} & \tb{\textbf{0.483}} & \tb{\textbf{0.900}} & \tgr{0.975} & \tgr{1.000} \\ \hline
\multicolumn{1}{|c|}{}                            & reward  & \tgr{0.665}          & \tgr{0.587}          & \tgr{0.303} & \tgr{0.240}          & \tgr{0.445}          & \tgr{0.462}          & \tgr{0.672}          & \tgr{0.507}          & \tgr{0.311}          & \tgr{0.230}          & \tgr{0.175} & \tgr{0.111} \\
\multicolumn{1}{|c|}{}                            & cost    & \tgr{1.258}          & \tgr{1.811}          & \tgr{0.678} & \tgr{0.758}          & \tgr{1.485}          & \textbf{0.960}                        & \tgr{1.002}          & \tgr{1.438}          & \tgr{0.549}          & \tgr{1.341}          & \tgr{1.259} & \tgr{0.963} \\
\multicolumn{1}{|c|}{\multirow{-3}{*}{DT+rand}}   & failure & \tgr{0.900}          & \tgr{1.000}          & \tgr{0.767} & \tgr{0.875}          & \tgr{1.000}          & \tgr{0.950}          & \tgr{0.867}          & \tgr{0.975}          & \tgr{0.617}          & \tgr{0.950}          & \tgr{0.975} & \tgr{0.925} \\ \hline
\multicolumn{1}{|c|}{}                            & reward  & \tgr{0.644}          & \tgr{0.521}          & \tgr{0.271} & \tgr{0.200}          & \tgr{0.486}          & \tgr{0.441}          & \tgr{0.636}          & \tgr{0.512}          & \tgr{0.321}          & \tgr{0.206}          & \tgr{0.131} & \tgr{0.126} \\
\multicolumn{1}{|c|}{}                            & cost    & \textbf{0.990}                        & \tgr{2.152}          & \tgr{0.640} & \tgr{0.730}          & \tgr{1.808}          & \tgr{1.273}          & \tgr{1.034}          & \tgr{1.497}          & \tgr{0.574}          & \tgr{1.271}          & \tgr{1.103} & \tgr{0.911} \\
\multicolumn{1}{|c|}{\multirow{-3}{*}{DT+maxmax}} & failure & \tgr{0.775}          & \tgr{1.000}          & \tgr{0.767} & \tgr{0.783}          & \tgr{1.000}          & \tgr{0.950}          & \tgr{0.817}          & \tgr{0.933}          & \tgr{0.625}          & \tgr{0.975}          & \tgr{0.967} & \tgr{0.975} \\ \hline
\end{tabular}%
}}}
\label{tab:sas}
\end{center}
\vskip -20pt
\end{table*}

\subsection{Results for all the datasets} 
We present the results for a total of 16 datasets in \cref{tab:cdt_Appendix}. 
These results include an additional experiment on four Circle tasks (PointCircle1, PointCircle2, CarCircle1, CarCircle2) and eight tasks in bullet-safety-gym environment(BallRun, CarRun, DroneRun, AntRun, BallCircle, CarCircle, DroneCircle, AntCircle). 
In PC2, CC1, and CC2 environments, CDT+SAS exhibited the highest reward among safe agents. 
CDT+SAS demonstrates lower costs than CDT in all four environments.

\begin{table*}[hbt!]
\caption{Complete evaluation results of the baselines and the Decision Transformer with our method (DT+SAS) and Constrained Decision Transformer with our method (CDT+SAS) in the Safety Gymnasium environment. The values are averaged across three different cost thresholds, 20 evaluation episodes, and three random seeds. \tgr{Gray}: Unsafe agents. \textbf{Bold}: Safe agents whose normalized cost is less than 1. \tb{\textbf{Blue}}: Agents which has highest reward among safe agents}
\label{tab:cdt_Appendix}
\vskip 0.15in

\begin{center}

   \resizebox{\textwidth}{!}{%
{\renewcommand{\arraystretch}{1.2}

\begin{tabular}{|cl|cc|cc|cc|cc|cc|cc|cc|cc|cc|}
\hline
\multicolumn{2}{|c|}{}                       & \multicolumn{2}{c|}{DT + SAS}                                              & \multicolumn{2}{c|}{CDT + SAS}                                             & \multicolumn{2}{c|}{CDT}                                                    & \multicolumn{2}{c|}{BC-All}                                & \multicolumn{2}{c|}{BC-Safe}                                                & \multicolumn{2}{c|}{BCQ-Lag}                                                & \multicolumn{2}{c|}{BEAR-Lag}                              & \multicolumn{2}{c|}{CPQ}                                   & \multicolumn{2}{c|}{COptiDICE}                             \\ \cline{3-20} 
\multicolumn{2}{|c|}{\multirow{-2}{*}{Task}} & reward                               & cost                                 & reward                               & cost                                 & reward                               & cost                                 & reward                       & cost                        & reward                               & cost                                 & reward                               & cost                                 & reward                       & cost                        & reward                      & cost                         & reward                       & cost                        \\ \hline
\multicolumn{2}{|c|}{PointGoal1}             & \tgr{0.66}          & \tgr{1.19}          & \tgr{0.65}          & \tgr{1.27}          & \tgr{0.69}          & \tgr{1.12}          & \textbf{0.65}                & \textbf{0.95}               & \textbf{0.43}                        & \textbf{0.54}                        & \tb{\textbf{0.71}} & \tb{\textbf{0.98}} & \tgr{0.74}  & \tgr{1.18} & \textbf{0.57}               & \textbf{0.35}                & \tgr{0.49}  & \tgr{1.66} \\
\multicolumn{2}{|c|}{PointGoal2}             & \tgr{0.65}          & \tgr{1.78}          & \tb{\textbf{0.52}} & \tb{\textbf{0.94}} & \tgr{0.59}          & \tgr{1.34}          & \tgr{0.54}  & \tgr{1.97} & \textbf{0.29}                        & \textbf{0.78}                        & \tgr{0.67}          & \tgr{3.18}          & \tgr{0.67}  & \tgr{3.11} & \tgr{0.4}  & \tgr{1.31}  & \tgr{0.38}  & \tgr{1.92} \\
\multicolumn{2}{|c|}{PointPush1}             & \textbf{0.28}                        & \textbf{0.62}                        & \textbf{0.26}                        & \textbf{0.54}                        & \textbf{0.24}                        & \textbf{0.48}                        & \textbf{0.19}                & \textbf{0.61}               & \textbf{0.13}                        & \textbf{0.43}                        & \tb{\textbf{0.33}} & \tb{\textbf{0.86}} & \textbf{0.22}                & \textbf{0.79}               & \textbf{0.2}                & \textbf{0.83}                & \textbf{0.13}                & \textbf{0.83}               \\
\multicolumn{2}{|c|}{PointPush2}             & \tb{\textbf{0.24}} & \tb{\textbf{0.64}} & \textbf{0.20}                        & \textbf{0.53}                        & \textbf{0.21}                        & \textbf{0.65}                        & \textbf{0.18}                & \textbf{0.91}               & \textbf{0.11}                        & \textbf{0.8}                         & \textbf{0.23}                        & \textbf{0.99}                        & \tgr{0.16}  & \tgr{0.89} & \tgr{0.11} & \tgr{1.04}  & \tgr{0.02}  & \tgr{1.18} \\
\multicolumn{2}{|c|}{PointButton1}           & \tgr{0.49}          & \tgr{1.38}          & \tgr{0.51}          & \tgr{1.27}          & \tgr{0.5}           & \tgr{1.68}          & \tgr{0.1}   & \tgr{10.5} & \tb{\textbf{0.06}} & \tb{\textbf{0.52}} & \tgr{0.24}          & \tgr{1.73}          & \tgr{0.2}   & \tgr{1.6}  & \tgr{0.69} & \tgr{3.2}   & \tgr{0.13}  & \tgr{1.4}  \\
\multicolumn{2}{|c|}{PointButton2}           & \tgr{0.51}          & \tgr{1.14}          & \tb{\textbf{0.41}} & \tb{\textbf{0.98}} & \tgr{0.46}          & \tgr{1.57}          & \tgr{0.27}  & \tgr{2.02} & \tgr{0.16}          & \tgr{1.1}           & \tgr{0.4}           & \tgr{2.66}          & \tgr{0.43}  & \tgr{2.47} & \tgr{0.58} & \tgr{4.3}   & \tgr{0.15}  & \tgr{1.51} \\
\multicolumn{2}{|c|}{PointCircle1}           & \tgr{0.69}          & \tgr{1.81}          & \textbf{0.54}                        & \textbf{0.21}                        & \tb{\textbf{0.59}} & \tb{\textbf{0.69}} & \tgr{0.79}  & \tgr{3.98} & \textbf{0.41}                        & \textbf{0.16}                        & \tgr{0.54}          & \tgr{2.38}          & \tgr{0.73}  & \tgr{3.28} & \textbf{0.43}               & \textbf{0.75}                & \tgr{0.86}  & \tgr{5.51} \\
\multicolumn{2}{|c|}{PointCircle2}           & \tgr{0.42}          & \tgr{1.69}          & \tb{\textbf{0.63}} & \tb{\textbf{0.47}} & \tgr{0.64}          & \tgr{1.05}          & \tgr{0.66}  & \tgr{4.17} & \textbf{0.48}                        & \textbf{0.99}                        & \tgr{0.66}          & \tgr{2.6}           & \tgr{0.63}  & \tgr{4.27} & \tgr{0.24} & \tgr{3.58}  & \tgr{0.85}  & \tgr{8.61} \\
\multicolumn{2}{|c|}{CarGoal1}               & \tb{\textbf{0.67}} & \tb{\textbf{0.85}} & \textbf{0.65}                        & \textbf{0.90}                        & \tgr{0.66}          & \tgr{1.21}          & \textbf{0.39}                & \textbf{0.33}               & \textbf{0.24}                        & \textbf{0.28}                        & \textbf{0.47}                        & \textbf{0.78}                        & \tgr{0.61}  & \tgr{1.13} & \tgr{0.79} & \tgr{1.42}  & \textbf{0.35}                & \textbf{0.54}               \\
\multicolumn{2}{|c|}{CarGoal2}               & \tgr{0.48}          & \tgr{1.15}          & \tb{\textbf{0.42}} & \tb{\textbf{0.98}} & \tgr{0.48}          & \tgr{1.25}          & \tgr{0.23}  & \tgr{1.05} & \textbf{0.14}                        & \textbf{0.51}                        & \tgr{0.3}           & \tgr{1.44}          & \tgr{0.28}  & \tgr{1.01} & \tgr{0.65} & \tgr{3.75}  & \textbf{0.25}                & \textbf{0.91}               \\
\multicolumn{2}{|c|}{CarPush1}               & \tb{\textbf{0.31}} & \tb{\textbf{0.51}} & \tb{\textbf{0.31}} & \tb{\textbf{0.49}} & \tb{\textbf{0.31}} & \tb{\textbf{0.4}}  & \textbf{0.22}                & \textbf{0.36}               & \textbf{0.14}                        & \textbf{0.33}                        & \textbf{0.23}                        & \textbf{0.43}                        & \textbf{0.21}                & \textbf{0.54}               & \textbf{-0.03}              & \textbf{0.95}                & \textbf{0.23}                & \textbf{0.5}                \\
\multicolumn{2}{|c|}{CarPush2}               & \tgr{0.22}          & \tgr{1.16}          & \tb{\textbf{0.21}} & \tb{\textbf{0.75}} & \tgr{0.19}          & \tgr{1.3}           & \textbf{0.14}                & \textbf{0.9}                & \textbf{0.05}                        & \textbf{0.45}                        & \tgr{0.15}          & \tgr{1.38}          & \tgr{0.1}   & \tgr{1.2}  & \tgr{0.24} & \tgr{4.25}  & \tgr{0.09}  & \tgr{1.07} \\
\multicolumn{2}{|c|}{CarButton1}             & \tgr{0.17}          & \tgr{1.08}          & \tb{\textbf{0.27}} & \tb{\textbf{0.98}} & \tgr{0.21}          & \tgr{1.6}           & \tgr{0.03}  & \tgr{1.38} & \textbf{0.07}                        & \textbf{0.85}                        & \tgr{0.04}          & \tgr{1.63}          & \tgr{0.18}  & \tgr{2.72} & \tgr{0.42} & \tgr{9.66}  & \tgr{-0.08} & \tgr{1.68} \\
\multicolumn{2}{|c|}{CarButton2}             & \tb{\textbf{0.14}} & \tb{\textbf{0.84}} & \tgr{0.30}          & \tgr{1.11}          & \tgr{0.13}          & \tgr{1.58}          & \tgr{-0.13} & \tgr{1.24} & \textbf{-0.01}                       & \textbf{0.63}                        & \tgr{0.06}          & \tgr{2.13}          & \tgr{-0.01} & \tgr{2.29} & \tgr{0.37} & \tgr{12.51} & \tgr{-0.07} & \tgr{1.59} \\
\multicolumn{2}{|c|}{CarCircle1}             & \tgr{ 0.41}          & \tgr{ 1.84}          & \tb{\textbf{0.47}} & \tb{\textbf{0.52}} & \tgr{ 0.6}           & \tgr{ 1.73}          & \tgr{ 0.72}  & \tgr{ 4.39} & \tgr{ 0.37}          & \tgr{ 1.38}          & \tgr{ 0.73}          & \tgr{ 5.25}          & \tgr{ 0.76}  & \tgr{ 5.46} & \tgr{ 0.02} & \tgr{ 2.29}  & \tgr{ 0.7}   & \tgr{ 5.72} \\
\multicolumn{2}{|c|}{CarCircle2}             & \tgr{ 0.63}          & \tgr{ 1.69}          & \tb{\textbf{0.56}} & \tb{\textbf{0.62}} & \tgr{ 0.66}          & \tgr{ 2.53}          & \tgr{ 0.76}  & \tgr{ 6.44} & \tgr{ 0.54}          & \tgr{ 3.38}          & \tgr{ 0.72}          & \tgr{ 6.58}          & \tgr{ 0.74}  & \tgr{ 6.82} & \tgr{ 0.44} & \tgr{ 2.69}  & \tgr{ 0.77}  & \tgr{ 7.99} \\ \hline
\multicolumn{2}{|c|}{BallRun}                & \tgr{0.99} & \tgr{1.6}  & \tb{\textbf{0.04}} & \tb{\textbf{0.29}} & \tgr{0.39}          & \tgr{1.16}          & \tgr{0.6}  & \tgr{5.08} & \tgr{0.27}         & \tgr{1.46}          & \tgr{0.76} & \tgr{3.91} & \tgr{-0.47} & \tgr{5.03} & \tgr{0.22}          & \tgr{1.27}          & \tgr{0.59} & \tgr{3.52} \\
\multicolumn{2}{|c|}{CarRun}                 & \tgr{8.12} & \tgr{1.06} & \textbf{0.72}                        & \textbf{0.39}                        & \tb{\textbf{0.99}} & \tb{\textbf{0.65}} & \textbf{0.97}               & \textbf{0.33}               & \textbf{0.94}                       & \textbf{0.22}                        & \textbf{0.94}               & \textbf{0.15}               & \tgr{0.68}  & \tgr{7.78} & \tgr{0.95}          & \tgr{1.79}          & \textbf{0.87}               & \textbf{0}                  \\
\multicolumn{2}{|c|}{DroneRun}               & \tgr{0.76} & \tgr{1.58} & \textbf{0.33}                        & \textbf{0.78}                        & \tb{\textbf{0.63}} & \tb{\textbf{0.79}} & \tgr{0.24} & \tgr{2.13} & \textbf{0.28}                       & \textbf{0.74}                        & \tgr{0.72} & \tgr{5.54} & \tgr{0.42}  & \tgr{2.47} & \tgr{0.33}          & \tgr{3.52}          & \tgr{0.67} & \tgr{4.15} \\
\multicolumn{2}{|c|}{AntRun}                 & \tgr{1.08} & \tgr{2.43} & \textbf{0.32}                        & \textbf{0.14}                        & \tb{\textbf{0.72}} & \tb{\textbf{0.91}} & \tgr{0.72} & \tgr{2.93} & \tgr{0.65}         & \tgr{1.09}          & \tgr{0.76} & \tgr{5.11} & \textbf{0.15}                & \textbf{0.73}               & \textbf{0.03}                        & \textbf{0.02}                        & \textbf{0.61}               & \textbf{0.94}               \\
\multicolumn{2}{|c|}{BallCircle}             & \tgr{0.81} & \tgr{1.41} & \textbf{0.32}                        & \textbf{0.38}                        & \tgr{0.77}          & \tgr{1.07}          & \tgr{0.74} & \tgr{4.71} & \textbf{0.52}                       & \textbf{0.65}                        & \tgr{0.69} & \tgr{2.36} & \tgr{0.86}  & \tgr{3.09} & \tb{\textbf{0.64}} & \tb{\textbf{0.76}} & \tgr{0.7}  & \tgr{2.61} \\
\multicolumn{2}{|c|}{CarCircle}              & \tgr{0.85} & \tgr{1.76} & \textbf{0.19}                        & \textbf{0.22}                        & \tb{\textbf{0.75}} & \tb{\textbf{0.95}} & \tgr{0.58} & \tgr{3.74} & \textbf{0.5}                        & \textbf{0.84}                        & \tgr{0.63} & \tgr{1.89} & \tgr{0.74}  & \tgr{2.18} & \textbf{0.71}                        & \textbf{0.33}                        & \tgr{0.49} & \tgr{3.14} \\
\multicolumn{2}{|c|}{DroneCircle}            & \tgr{0.82} & \tgr{1.55} & \textbf{0.51}                        & \textbf{0.42}                        & \tb{\textbf{0.63}} & \tb{\textbf{0.98}} & \tgr{0.72} & \tgr{3.03} & \textbf{0.56}                       & \textbf{0.57}                        & \tgr{0.8}  & \tgr{3.07} & \tgr{0.78}  & \tgr{3.68} & \tgr{-0.22}         & \tgr{1.28}          & \tgr{0.26} & \tgr{1.02} \\
\multicolumn{2}{|c|}{AntCircle}              & \tgr{0.59} & \tgr{1.18} & \textbf{0.26}                        & \textbf{0.34}                        & \tgr{0.54}          & \tgr{1.78}          & \tgr{0.58} & \tgr{4.9}  & \tb{\textbf{0.4}} & \tb{\textbf{0.96}} & \tgr{0.58} & \tgr{2.87} & \tgr{0.65}  & \tgr{5.48} & \textbf{0}                           & \textbf{0}                           & \tgr{0.17} & \tgr{5.04} \\ \hline
\end{tabular}%
}
}

\end{center}
\end{table*}

\begin{table}[hbt!]
\caption{The modified version of \cref{tab:sas} with standard deviation across 3 cost thresholds, 20 evaluation episodes, and 3 random seeds.}
\label{tab:sas_modified}
   \resizebox{\textwidth}{!}{%
   \renewcommand{\arraystretch}{1.2}
\begin{tabular}{|cl|cccc|cccc|cccc|cccc|}
\hline
\multicolumn{2}{|c|}{\multirow{3}{*}{Task}} & \multicolumn{4}{c|}{CDT}                                                                    & \multicolumn{4}{c|}{CDT+SAS}                                                              & \multicolumn{4}{c|}{DT}                                                                  & \multicolumn{4}{c|}{DT+SAS}                                                             \\ \cline{3-18} 
\multicolumn{2}{|c|}{}                      & \multicolumn{2}{c|}{reward}                            & \multicolumn{2}{c|}{cost}          & \multicolumn{2}{c|}{reward}                            & \multicolumn{2}{c|}{cost}         & \multicolumn{2}{c|}{reward}                           & \multicolumn{2}{c|}{cost}        & \multicolumn{2}{c|}{reward}                           & \multicolumn{2}{c|}{cost}        \\ \cline{3-18} 
\multicolumn{2}{|c|}{}                      & \multicolumn{1}{c|}{mean} & \multicolumn{1}{c|}{std}   & \multicolumn{1}{c|}{mean} & std    & \multicolumn{1}{c|}{mean} & \multicolumn{1}{c|}{std}   & \multicolumn{1}{c|}{mean} & std   & \multicolumn{1}{c|}{mean} & \multicolumn{1}{c|}{std}  & \multicolumn{1}{c|}{mean} & std  & \multicolumn{1}{c|}{mean} & \multicolumn{1}{c|}{std}  & \multicolumn{1}{c|}{mean} & std  \\ \hline
\multicolumn{2}{|c|}{PointGoal1}            & 0.69 & 0.007 & 1.12 & 0.037  & 0.65 & 0.007 & 1.27 & 0.062 & 0.66 & 0.02 & 1.32 & 0.31 & 0.66 & 0.03 & \textbf{1.19} & 0.15 \\
\multicolumn{2}{|c|}{PointGoal2}            & 0.59 & 0.017 & 1.34 & 0.054  & 0.52 & 0.036 & \textbf{0.94} & 0.158 & 0.38 & 0.02 & 2.63 & 0.05 & 0.65 & 0.09 & 1.78 & 0.17 \\
\multicolumn{2}{|c|}{PointPush1}            & 0.24 & 0.012 & 0.48 & 0.023  & 0.26 & 0.027 & 0.54 & 0.019 & 0.22 & 0.06 & 0.93 & 0.21 & 0.28 & 0.01 & \textbf{0.62} & 0.10 \\
\multicolumn{2}{|c|}{PointPush2}            & 0.21 & 1.363 & 0.65 & 31.063 & 0.20 & 0.038 & \textbf{0.53} & 0.089 & 0.20 & 0.08 & 0.78 & 0.45 & 0.24 & 0.06 & \textbf{0.64} & 0.09 \\
\multicolumn{2}{|c|}{PointButton1}          & 0.5  & 0.006 & 1.68 & 0.049  & 0.51 & 0.026 & \textbf{1.27} & 0.044 & 0.38 & 0.04 & 1.19 & 0.18 & 0.49 & 0.05 & 1.38 & 0.21 \\
\multicolumn{2}{|c|}{PointButton2}          & 0.46 & 0.019 & 1.57 & 0.047  & 0.41 & 0.019 & \textbf{0.98} & 0.026 & 0.50 & 0.06 & 1.31 & 0.14 & 0.51 & 0.00 & \textbf{1.14} & 0.13 \\
\multicolumn{2}{|c|}{CarGoal1}              & 0.66 & 0.008 & 1.21 & 0.057  & 0.65 & 0.008 & \textbf{0.90} & 0.035 & 0.64 & 0.02 & 0.98 & 0.12 & 0.67 & 0.03 & \textbf{0.85} & 0.16 \\
\multicolumn{2}{|c|}{CarGoal2}              & 0.48 & 0.032 & 1.25 & 0.095  & 0.42 & 0.032 & \textbf{0.98} & 0.047 & 0.51 & 0.04 & 1.47 & 0.32 & 0.48 & 0.03 & \textbf{1.15} & 0.20 \\
\multicolumn{2}{|c|}{CarPush1}              & 0.31 & 0.018 & 0.40 & 0.068  & 0.31 & 0.018 & \textbf{0.49} & 0.097 & 0.35 & 0.07 & 0.68 & 0.22 & 0.31 & 0.01 & \textbf{0.51} & 0.15 \\
\multicolumn{2}{|c|}{CarPush2}              & 0.19 & 0.022 & 1.30 & 0.081  & 0.21 & 0.023 & \textbf{0.75} & 0.120 & 0.20 & 0.03 & 1.17 & 0.26 & 0.22 & 0.01 & \textbf{1.16} & 0.26 \\
\multicolumn{2}{|c|}{CarButton1}            & 0.21 & 0.014 & 1.60 & 0.106  & 0.27 & 0.081 & \textbf{0.98} & 0.006 & 0.24 & 0.04 & 1.42 & 0.04 & 0.17 & 0.03 & \textbf{1.08} & 0.17 \\
\multicolumn{2}{|c|}{CarButton2}            & 0.13 & 0.031 & 1.58 & 0.034  & 0.30 & 0.009 & \textbf{1.11} & 0.025 & 0.21 & 0.04 & 1.05 & 0.21 & 0.14 & 0.03 & \textbf{0.84} & 0.08 \\ \hline
\end{tabular}}
\end{table}

\subsection{Additional comparison with SOTA offline RL methods and offline meta-RL} 
We note that our DT+SAS which uses the pretrained DT without cost training data outperforms the above SOTA offline safe RL methods.
Moreover, we provide the comparison with CQL, SAC-n, and APE-V which is the online (few-shot) adaptation method for the offline RL algorithm in the table below. We note that our method shows the better improvement compared to the reported value of SAC-n → APE-V in APE-V paper.
However, the target task of offline meta-RL focuses on the adaptation performance when the goal of the target task changes significantly, which differs critically from measuring the generalization performance that is the aim of our paper, making it challenging to conduct additional experiments.

\begin{table}[hbt!]
\caption{Experiment results with CQL algorithm \citep{kumar2020conservative} and APE-V algorithm \citep{ghosh2022offline} in D4RL \citep{fu2021drl} datasets.}
\label{tab:cql_ape-v}

\vskip 0.15in
   \resizebox{\columnwidth}{!}{%
   {\renewcommand{\arraystretch}{1.0}
\begin{tabular}{cll|c|ccccc|ccc}
\hline
\multicolumn{3}{c|}{\multirow{2}{*}{Task Name}} & CQL    & \multicolumn{2}{c}{DT} & \multicolumn{2}{c}{DT+SAS} &            & SAC-N  & APE-V  &            \\ \cline{4-12} 
\multicolumn{3}{c|}{}                           & reward & reward    & failure    & reward       & failure      & improve(\%) & reward & reward & improve(\%) \\ \hline
\multicolumn{3}{c|}{hopper-medium-expert}       & 96.9   & 111.8     & 0.1        & 110.4        & 0.05         & -1.25      & 110    & 105.7  & -3.91      \\
\multicolumn{3}{c|}{hopper-medium-replay}       & 86.3   & 94.3      & 0          & 97.3         & 0            & 3.18       & 101.8  & 98.5   & -3.24      \\
\multicolumn{3}{c|}{walker2d-medium-expert}     & 109.1  & 108.3     & 0          & 107.5        & 0            & -0.74      & 116    & 110    & -5.17      \\
\multicolumn{3}{c|}{walker2d-medium-replay}     & 76.8   & 43.9      & 1          & 69.1         & 0.6          & 57.4       & 78.7   & 82.9   & 5.34       \\ \hline
\end{tabular}%
}}
\end{table}



\section{PROOF}\label{sec:proof}

We first provide technical results in the main paper.
By defining the condition probability of prompt $p_\mathbf{1:K}$ given high-level policy $\pi^{\textnormal{high}}_{\theta}$, we leverage $r_K(\theta)$ to make sure that the well-designed prompt is selected when it is from underlying the  safe high-level policy $\pi^{\textnormal{high}}_{\theta^*}$.
In details, the length variable $K$ can be composed of two conditions, the length of prompt and the number of prompt.
We note that we can have high probability of $p(\mathcal{O}_t =1| \mathbf{z}_t)= \exp( r(\pi^{\textnormal{high}}_{\theta}))$ when we provide the most matching prompt $\mathbf{p}^*$ with the underlying $\pi^{\textnormal{high}}_{\theta^*}$.

\subsection{Connection of density and safe RL}
Let $U$ be the universal set of state-action space and $B$ be the set $\{(s,a)|(s,a),  C(s,a)<C_{th} \}$ where $C_{th}$ is the some threshold of cost which satisfies $C_{th} \leq d (1-\gamma)$.
Assume that the cost value is bounded as $0\leq C(s,a) \leq C_{max}$.
We define a volume constant $\alpha=\E_{(s,a)\sim U}[\1((s,a)\in B)]$, which is $0<\alpha<1$.
Then, we can write the above inequality as
\begin{align}
    &J_C(\pi)-d  =\E_{(s,a) \sim B} \left[\hat{\rho}^{\pi}(s,a) C(s,a) \right] -\alpha d \nonumber \\
    & \quad\quad+ \E_{(s,a) \sim B^C} \left[\hat{\rho}^{\pi}(s,a) C(s,a) \right] -(1-\alpha)d \nonumber\\ 
    &\leq\E_{(s,a) \sim B} \left[\hat{\rho}^{\pi}(s,a) C_{th} -d \right] \nonumber \\
    & \quad\quad+ \E_{(s,a) \sim B^C} \left[\hat{\rho}^{\pi}(s,a) C_{max}-d \right]. \label{eq:occ_measure_const}
\end{align}
We can observe that expert policies in the offline RL dataset $\gD$ should have low values in $B^C$ to satisfy the constraint inequality of \cref{eq:occ_measure_const} less than zero.
Now, we know that reducing the marginal value of occupancy measure over $B^C$ leads to the lower bound of the cost function of $\pi$.
Suppose that $\hat{\rho}^\pi(s,a)\leq {d \over C_{max}}$ for $(s,a)\in B^C$.
We consider the definition of occupancy measure and assume that occupancy measure is a continuous function.
We have  ${d\over C_{max} }\leq\hat{\rho}^\pi(s,a)\leq {d\over C_{th} } \leq {1 \over 1-\gamma}$ for $(s,a)\in B$, and then get $J_C(\pi)-d\leq 0$.
This is an intuitive condition to be an expert policy trained by the constrained RL in \cref{eq:cmdp}.

\subsection{Proof of \cref{eq:in_context_predictor_uvc}}
To show the derivation, we start from
\begin{align}
    p(\tau \mid p_{1:K}, s_1^{\textnormal{test}})
    = \int_{\theta} p(\tau \mid p_{1:K}, s_1^{\textnormal{test}}, \theta)\, p(\theta)\, d\theta.
\end{align}

To check the optimality between the generated trajectory and the prompt, we prove
\begin{align}
    p(\mathcal{O}_{\textnormal{traj}} \mid p_{1:K}, s_1^{\textnormal{test}})
    = \int_{\theta} \sum_{z_1^{\textnormal{test}} \in \mathcal{Z}}
    \Bigg( g_{\pi_\theta}(\tau, z_1^{\textnormal{test}})
      \prod_{t=1} p(\mathcal{O}_t \mid s_t^{\textnormal{test}}, a_t^{\textnormal{test}}) \Bigg)
   p(z_1^\text{test}|\mathbf{p}_{1:K}, s_1^\text{test}, \theta)  e^{K \cdot r_K(\theta)}\, p(\theta)\, d\theta,
\end{align}
where
\begin{align}
    \sum_{z_1^{\textnormal{test}} \in \mathcal{Z}} \;
    \prod_{t=1} p(s_{t+1}^{\textnormal{test}} \mid s_t^{\textnormal{test}}, a_t^{\textnormal{test}})
    \underbrace{p(a_t^{\textnormal{test}} \mid s_t^{\textnormal{test}}, z_t^{\textnormal{test}})}_{\pi^{\textnormal{low}}_\phi}\;
    \underbrace{p_\theta(z_t^{\textnormal{test}} \mid s_t^{\textnormal{test}}, z_{t-1}^{\textnormal{test}})}_{\pi^{\textnormal{high}}_\theta}
    \;=:\; \sum_{z_1^{\textnormal{test}} \in \mathcal{Z}} g_{\pi_\theta}(\tau, z_1^{\textnormal{test}}).
\end{align}

By Bayes’ rule and the law of total probability, we have
\begin{align}
    p(\mathcal{O}_{\textnormal{traj}} \mid p_{1:K}, s_1^{\textnormal{test}})
    &= \int_{\theta} p(\tau \mid p_{1:K}, s_1^{\textnormal{test}}, \theta)\;
       p(\theta \mid p_{1:K}, s_t^{\textnormal{test}})\, d\theta \nonumber \\
    &\propto \int_{\theta} p(\tau \mid p_{1:K}, s_1^{\textnormal{test}}, \theta)\;
       p(p_{1:K}, s_t^{\textnormal{test}} \mid \theta)\, p(\theta)\, d\theta \nonumber \\
    &= \int_{\theta} \sum_{z_1^{\textnormal{test}} \in \mathcal{Z}}
       \Bigg( g_{\pi_\theta}(\tau, z_1^{\textnormal{test}})
       \prod_{t=1} p(\mathcal{O}_t \mid s_t^{\textnormal{test}}, a_t^{\textnormal{test}}) \Bigg)p(z_1^\text{test}|\mathbf{p}_{1:K}, s_1^\text{test}, \theta) 
       \frac{p(p_{1:K}, s_t^{\textnormal{test}} \mid \theta)}
            {p(p_{1:K}, s_t^{\textnormal{test}} \mid \theta^*)}\, p(\theta)\, d\theta \nonumber \\
    &= \int_{\theta} \sum_{z_1^{\textnormal{test}} \in \mathcal{Z}}
       \Bigg( g_{\pi_\theta}(\tau, z_1^{\textnormal{test}})
       \prod_{t=1} p(\mathcal{O}_t \mid s_t^{\textnormal{test}}, a_t^{\textnormal{test}}) \Bigg)p(z_1^\text{test}|\mathbf{p}_{1:K}, s_1^\text{test}, \theta) 
       \exp\!\big(K \cdot r_K(\theta)\big)\, p(\theta)\, d\theta.
\end{align}

By the definition of $r_K(\theta)$, under distinguishability for all
$\pi^{\textnormal{high}}_\theta \neq \pi^{\textnormal{high}}_{\theta^*}$,
$r_K(\theta)$ converges to a negative constant, and letting $K \to \infty$ gives
$\exp(r_K(\pi^{\textnormal{high}}_\theta))=0$ for all
$\pi^{\textnormal{high}}_\theta \neq \pi^{\textnormal{high}}_{\theta^*}$,
and $\exp(r_K(\pi^{\textnormal{high}}_\theta))=1$ for
$\pi^{\textnormal{high}}_\theta = \pi^{\textnormal{high}}_{\theta^*}$.
A detailed discussion of distinguishability can be found in \citet{xie2021explanation}.

Moreover, the graphical model contains $p(z_t \mid s_t, z_{t-1})$, 
which samples latent skill variables when $s_t$ is given.
Thus, the transformer with latent variables can be viewed as hierarchical RL
with high-level policy $\pi^{\textnormal{high}}_\theta = p(z_t \mid s_t, z_{t-1})$.

Hence, interpreting the transformer as implicit Bayesian inference of in-context learning \citep{xie2021explanation},
we obtain a safe high-level policy when the sampled trajectory instruction in \cref{alg:main}
satisfies Lyapunov conditions at every step.
Consequently, the in-context learner can act at the test-time initial state
according to a Lyapunov-stable policy.

\subsection{Proof of \cref{thm:lyapunov_gm}}
Since $\mathcal{U}_t$ and $\mathcal{V}_t$ are both optimality variables indicating the Lyapunov condition,
we apply probabilistic inference for RL as
\begin{align}
    \log p(\mathcal{U}_{1:T}, \mathcal{V}_{1:T} \mid \tau)
    &= \log \Bigg( p(s_1) \prod_{t=1}^T 
        p(\mathcal{U}_{t}, \mathcal{V}_{t} \mid s_t, a_t)\,
        p(s_{t+1}\mid s_t, a_t)\,
        p(a_t \mid s_t, z_t)\,
        p(z_t \mid s_t, z_{t-1}) \Bigg) \nonumber \\
    &= \sum_{t=1}^T \log p(\mathcal{U}_{t}, \mathcal{V}_{t} \mid s_t, a_t) + C.
\end{align}

When all $\mathcal{U}_t=\mathcal{V}_t=1$, the trajectory perfectly satisfies the Lyapunov condition.

Recall that a trajectory is asymptotically stable if
\begin{equation}\label{eq:lyapunov_condition}
    \textbf{(1)} \;\; G(s_e,a_e)=0, 
    \quad \textbf{(2)} \;\; G(s_t,a_t)>0,\;\forall (s_t,a_t)\neq(s_e,a_e), 
    \quad \textbf{(3)} \;\; G(s_t,a_t)\ge G(s_{t+1},a_{t+1}).
\end{equation}

We define our Lyapunov function as
\begin{equation}
    G_\textnormal{SAS}(s_t,a_t) 
    = \min_\pi \max_{t'} E(s_{t'},\pi(s_{t'})) - E(s_t,a_t),
\end{equation}
so that the equilibrium point satisfies $G_\textnormal{SAS}(s_e,a_e)=0$.
Thus, $\mathcal{U}_t=1$ corresponds to condition (2), 
and $\mathcal{V}_t=1$ corresponds to condition (3).

If we choose the distributions
\begin{align}
    p(\mathcal{U}_t=1 \mid s_t,a_t) &\propto 
        \exp\!\left(\1 [\,G_\textnormal{SAS}(s_t,a_t) > 0\,]\right), \\
    p(\mathcal{V}_t=1 \mid s_t,a_t) &\propto 
        \exp\!\left(\1 [\,G_\textnormal{SAS}(s_t,a_t) - G_\textnormal{SAS}(s_{t+1},a_{t+1}) \ge 0\,]\right),
\end{align}
then
\begin{align}
    \sum_{t=1}^T \log p(\mathcal{U}_t,\mathcal{V}_t \mid s_t,a_t)
    &= \sum_{t=1}^T \log p(\mathcal{V}_t \mid s_t,a_t,\mathcal{U}_t)\,p(\mathcal{U}_t \mid s_t,a_t) \nonumber \\
    &\propto \sum_{t=1}^T \1[\,G_\textnormal{SAS}(s_t,a_t) > 0\,] 
       + \sum_{t=1}^T \1[\,G_\textnormal{SAS}(s_t,a_t) - G_\textnormal{SAS}(s_{t+1},a_{t+1}) \ge 0\,].
\end{align}

Maximizing the above implies that the trajectory approaches the Lyapunov condition.

\subsection{Proof of \cref{thm:control_invariant}}\label{proof:hoeffding}

The goal of our method is to keep occupancy measures in the distribution of the target control-invariant set $\mathcal{R}=\{ (s_t,a_t) | c_1 \leq E(s_t, a_t) \leq c_2\}$ where $E(s_t,a_t) = - \log \rho(s_t,a_t)$ for utilizing the pretrained expert distribution.
As our Lyapunov function approximation is defined as
\begin{align*}
G(s_t,a_t) &= \underset{i=1,\cdots,N}{\min} \underset{j=1,\cdots T}{\max} E(s_j, \pi_i(s_j)) - E(s_t,a_t)
\end{align*}
for $N$ sample trajectories with the episode length $T$ in the first loop of Algorithm 1.
Suppose that $c_2$ is some constant that is larger than $\underset{i=1,\cdots,N}{\min} \underset{j=1,\cdots T}{\max} E(s_j, \pi_i(s_j))$ for any $N,T$.
For clarity, we assume that the state–action pairs 
$\{(s_t,a_t)\}_{t=1}^T$ are sampled i.i.d.\ from the dataset $\mathcal{D}$.   
In practice, trajectories in an MDP exhibit temporal dependence; 
however, the same arguments extend under standard $\beta$-mixing or 
ergodicity conditions, which yield similar bounds up to constant factors.  
The proof proceeds in two steps:
\begin{enumerate}
    \item \textbf{Bounding high-energy events.}  
    Using Markov’s inequality, we control the probability that a sampled 
    trajectory includes pairs with energy $E(s,a) = -\log \hat{\rho}(s,a)$ 
    exceeding the threshold $c_2$.  
    This yields the first term in \eqref{eq:control_invariant}.
    \item \textbf{Ensuring Lyapunov monotonicity.}  
    Applying Hoeffding’s inequality to the Lyapunov descent indicators 
    $\{\mathcal{V}_t\}$, we bound the probability that cumulative descent 
    violations exceed $\kappa(c_2-c_1)/L$.  
    This leads to the second exponential term.
\end{enumerate}
Combining both bounds yields the probability estimate in 
\eqref{eq:control_invariant}.  
We now demonstrate that Algorithm 1 reduces the probability of escaping from the control invariant set as the numbers of iterations, $N$ and $M$ for the first and second loops, respectively, increase.

\begin{assumption}\label{assumption}
    The difference $\lVert G(s_t,a_t) - G(s_{t+1}, a_{t+1})\rVert$ in Eq. (3) over the transition $\mathcal{T}$ is bounded as $ \lVert G(s_t,a_t) - G(s_{t+1}, a_{t+1})\rVert \leq L$ for all $t$.
\end{assumption}

\begin{proof}
    First note that $P\left[\tau \not\subset \mathcal{R} \right]$ is less than the sum of the probability of $E(s_t,a_t) \geq c_2$ for all data points in $N$ trajectories and the probability that all $M$ trials moves below $E(s_t,a_t)\leq c_1$. By using Markov's inequality for the first term of RHS and Hoeffding's inequality for the second term of RHS. Then, we have
    \begin{align*}
        P\left[\tau_\text{best} \not\subset \mathcal{R} \right] &\leq \left(P\left[ E(s,a) \geq c_2\right]\right)^{NT} + \left( P\left[ \sum_{t=1}^T \1( \mathcal{V}_t\neq 1) \geq {\kappa(c_2 - c_1) \over L} \right]\right)^M\\
        &\leq \left[ {\E_{(s,a) \sim \mathcal{D}} [E(s,a)] \over c_2} \right]^{NT} +  \exp\left(- {2 M \kappa^2 (c_2-c_1)^2 \over TL^2} \right)
    \end{align*}
    for some constant $\kappa$ to describe the average distance to escape.

    We note that the probability of $E(s_t,a_t) \geq c_2$ for all data points in $N$ trajectories implies $\1(\mathcal{U}_t \neq 1)$ because $G(s_t,a_t)=\underset{i=1,\cdots,N}{\min} \underset{j=1,\cdots T}{\max} E(s_j, \pi_i(s_j)) - E(s_t,a_t)<c_2 - E(s_t,a_t)<0$.

\end{proof}

\end{document}